\documentclass[msom,sglanonrev]{informs4} 
\RequirePackage{tgtermes}
\RequirePackage{newtxtext}
\RequirePackage{newtxmath}
\RequirePackage{bm}
\RequirePackage{endnotes}

\OneAndAHalfSpacedXII  
\usepackage{algorithm}
\usepackage{algpseudocode}
\usepackage{tikz}
\usepackage{booktabs}
\usepackage{placeins}

\makeatletter
\newenvironment{breakablealgorithm}
  {\par\addvspace{\medskipamount}\refstepcounter{algorithm}%
   \hrule height.8pt depth0pt \kern2pt%
   \renewcommand{\caption}[2][\relax]{%
     {\raggedright\textbf{\@ifundefined{fname@algorithm}{Algorithm}{\fname@algorithm}~\thealgorithm}\ ##2\par}%
     \kern2pt\hrule\kern2pt}}
  {\kern2pt\hrule\par\addvspace{\medskipamount}}
\makeatother

\definecolor{v5col}{rgb}{0,0,0} 
\definecolor{codexcomment}{rgb}{0,0,0}

\usepackage[normalem]{ulem}

\usepackage{natbib}
 \bibpunct[, ]{(}{)}{,}{a}{}{,}%
 \def\bibfont{\small}%
 \def\bibsep{\smallskipamount}%
 \def\bibhang{24pt}%

\renewcommand{\theARTICLETOP}{} 
\makeatletter
\def\theARTICLEABSTRACT{%
 \HOOKb
 \vspace*{18pt}
 \begin{minipage}[t]{\textwidth}\parindent1em
  \ABSfont
  \noindent\theABSTRACT\endgraf
  \vskip5pt
  \theFUNDING
  \theKEYWORDS
  \theSUBJECTCLASS
  \theAREAOFREVIEW
  \theMSCCLASS
  \theORMSCLASS
  \if@BLINDREV\else\theHISTORY\fi
 \noindent \hrulefill
 \end{minipage}%
 \vspace*{0pt}
 }
\makeatother

\EquationsNumberedThrough   
\TheoremsNumberedThrough    
\ECRepeatTheorems

\definecolor{linkblue}{RGB}{36,65,100}
\usepackage[colorlinks=true,linkcolor=linkblue,citecolor=linkblue,urlcolor=linkblue,filecolor=linkblue,pdfborder={0 0 0},hypertexnames=false,plainpages=false,bookmarksnumbered]{hyperref}
\makeatletter
\def\ps@ssrn{%
  \let\@oddhead\@empty
  \let\@evenhead\@empty
  \def\@oddfoot{\hfil{\normalfont\small\thepage}\hfil}%
  \let\@evenfoot\@oddfoot
}
\let\ps@headings\ps@ssrn
\let\ps@ECheadings\ps@ssrn
\let\ps@firstpage\ps@ssrn
\makeatother
\AtBeginDocument{\pagestyle{ssrn}}

\MANUSCRIPTNO{}

\begin{document}

\RUNAUTHOR{Bu, Tang, and Zhang}

\RUNTITLE{OR for AI That Does OR}

\TITLE{OR for AI That Does OR: Routing LLMs up the Escalator inside the OSCAR Framework}

\ARTICLEAUTHORS{%
\AUTHOR{Jinzhi Bu, Haixin Tang}
\AFF{Department of Logistics and Maritime Studies, The Hong Kong Polytechnic University, Hong Kong\\
\EMAIL{jinzhi.bu@polyu.edu.hk, haixin.tang@connect.polyu.hk}}

\AUTHOR{Huanan Zhang}
\AFF{Leeds School of Business, University of Colorado Boulder, Boulder, CO, 80309, USA\\
\EMAIL{Huanan.Zhang@colorado.edu}}
}  

\ABSTRACT{%
\textbf{Problem definition:} Large language models can translate business descriptions into optimization models, but executable code may misrepresent constraints or objectives. A solver can then return an optimal solution to the wrong problem. \textcolor{black}{Even when the solution satisfies the intended operating rules, a better plan may exist.} For organizations that repeatedly use optimization modeling, an LLM-based framework should produce accurate formulations at low cost and, ideally, run locally. We study how to verify improvements and allocate attempts across LLMs that differ in price and capability.

\textbf{Methodology/results:} We develop OSCAR (\textbf{O}ptimization modeling by \textbf{S}imulator, \textbf{C}oder, \textbf{A}nd \textbf{R}eviewer), which uses an offline Simulator certified against \textcolor{black}{labeled decision examples} to compare candidates and continues searching beyond feasibility. \textcolor{black}{We model the search for the next certified improvement as sequential decisions under unobserved difficulty: which LLMs to call and when to stop. In a simplified known-prior setting, we give conditions under which cost-ordered escalation is optimal.} For general menus, we derive a prior-free competitive guarantee. On five benchmark problems, OSCAR achieves 95\% to 100\% accuracy at the reported settings using \textcolor{black}{two small open-weight LLMs, each deployable locally on a single GPU}. Their single-attempt accuracies average 29\% and 48\%. In five runs per problem, Codex and Claude Code incur average token costs 3.1 and 5.8 times OSCAR's, respectively.

\textbf{Managerial implications:} Before upgrading to frontier models for optimization modeling, firms should assess what available models can achieve through verification and iterative improvement. \textcolor{black}{OSCAR supports open-weight models locally or in the cloud, depending on budget and confidentiality requirements. Firms should maintain labeled decision examples of feasible and infeasible decisions to clarify plain-language operating rules.} \textcolor{black}{OSCAR follows these labels when an LLM's interpretation conflicts with them.} \textcolor{black}{As LLM capabilities and prices change, OSCAR's simple operating rules and adjustable settings help firms adapt their model choices and benefit from these advances.}
}%

\KEYWORDS{large language models, optimization modeling, AI agents, LLM routing, token cost, OR for AI}

\maketitle

\section{Introduction}\label{sec:intro}

Large language models (LLMs), on their own, are not designed to replace optimization solvers for exact numerical computation or large-scale search over combinatorial solution spaces. \textcolor{black}{In optimization modeling,} OR researchers use \textcolor{black}{LLMs} to translate business problems described in plain language into formal models and solver-ready code, leaving numerical computation and combinatorial search to solvers such as COPT or Gurobi. Changes in business rules or problem structures can create a recurring need for optimization modeling. Errors in this translation can be difficult to detect. A model with a misread constraint or an incorrect objective may compile and solve to optimality, even though it represents a different problem.

Generating more candidate formulations does not resolve this difficulty. \textcolor{black}{Figure~\ref{fig:formulations} illustrates three formulations of the same description with different feasible regions. A solution that is optimal for one need not be feasible or optimal for another.} Their reported objective values are not directly comparable, and the apparently best solution may be infeasible in the real system. Improving on a single attempt requires a common basis for comparing the resulting solutions.

\begin{figure}[htbp]
\FIGURE
{\includegraphics[width=0.9\textwidth]{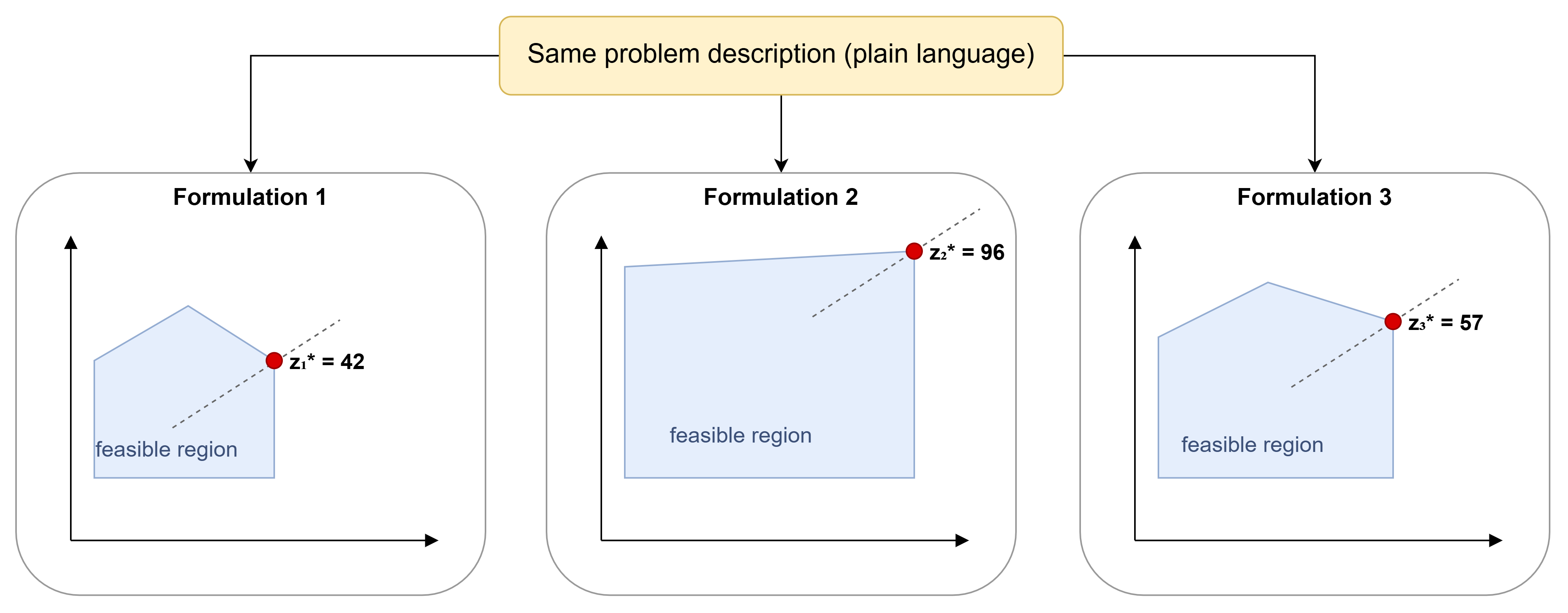}}
{Three candidate formulations from the same problem description. Each of them (or none of them) may be correct.\label{fig:formulations}}
{}
\end{figure}

\textcolor{black}{In the existing literature,} some  frameworks generate formulations without a post-execution validation-and-revision loop \textcolor{black}{\citep{huang2025orlm, liang2026lean}}. Others \textcolor{black}{revise formulations in response to execution feedback, including code errors or failure to obtain a feasible solution} \citep{xiao2024coe, ahmaditeshnizi2024optimus}, or check them through further LLM reasoning \citep{fang2026trival, liu2026optverifier}, solver certificates and domain checks \citep{ao2026optirepair}, or an executable simulator \citep{song2026nemo}. {OptMaster \textcolor{black}{constructed by \citet{lin2026optmaster}} also uses a simulator during continued search over a directed acyclic graph (DAG).} Yet a formulation that produces a feasible solution may exclude better solutions allowed by the problem description. Verification and validation of AI-generated models remain open research questions \citep{simchilevi2026democratizing}.

We develop OSCAR (\textbf{O}ptimization modeling by \textbf{S}imulator, \textbf{C}oder, \textbf{A}nd \textbf{R}eviewer), a framework that compares successive candidates and continues searching after feasibility is attained. Like \citet{song2026nemo}, OSCAR evaluates solutions with an executable offline Simulator. \textcolor{black}{In their framework}, a coding agent constructs the simulator and its tests from the problem description. OSCAR additionally uses \textcolor{black}{labeled decision examples: operational decisions with independently known feasibility labels and, when available, objective values.} \textcolor{black}{We use these additional data to test and iteratively revise the Simulator until it passes certification. Only then do we begin generating formulations, using the certified Simulator as a common basis for comparing candidates and identifying improvements.} Because further attempts \textcolor{black}{incur token costs} and models differ in price and capability, a routing policy chooses which models to call and decides when to stop.

We study a single target instance and, as in much of the literature \textcolor{black}{\citep{huang2025orlm, huang2024mamo, jiang2025llmopt}}, assess a formulation by \emph{solution correctness}: whether its solution is feasible and optimal for that instance in the real system. By contrast, \emph{formulation correctness} requires the model to be correct for every instance consistent with the description. We leave this stronger requirement to future research (Section~\ref{sec:conclusion}).

\subsection{Our Contributions}\label{sec:contrib}

The paper makes the following contributions.

\begin{itemize}
  \item \textbf{The OSCAR framework for improvement beyond feasibility.} We introduce OSCAR, which continues searching after a feasible solution is found. An offline Simulator compares candidate solutions so the framework can identify and retain improvements. \textcolor{black}{The Coder writes and revises formulations, while the Reviewer draws on the problem description, Simulator feedback, and its own judgment to recommend changes.} Once the Simulator is certified, the improvement loop requires no human intervention. OSCAR accommodates different tiers of LLMs, including inexpensive open models, models tailored to optimization, and highly capable general-purpose models. The framework allocates attempts according to their costs and capabilities.

  \item \textbf{The Escalator policy and its \textcolor{black}{theoretical} guarantees.} We model an improvement window, the search for the next certified improvement, as sequential search under an unobserved difficulty, with \textcolor{black}{given} per-call success probabilities and costs. \textcolor{black}{In a simplified setting with a known prior over difficulty, we establish conditions under which a cost-ordered escalation routing policy is optimal (Section~\ref{sec:prior}).} For a general menu without a prior, Algorithm~\ref{alg:det-escalator} computes \textcolor{black}{an Escalator schedule} whose expected loss is within a constant factor of a clairvoyant's \textcolor{black}{(Section~\ref{sec:competitive})},  
  \textcolor{black}{where} the factor depends only on success probabilities. Expected loss includes calling costs and a penalty for ending the window without improvement. 

  \item \textbf{Numerical evidence.} On five FrontierOR problems \citep{frontieror2026}, \textcolor{black}{our experiments show that} \textcolor{black}{two small and older LLMs} produce correct solutions in 29\% and 48\% of single attempts on average. \textcolor{black}{Each LLM can be deployed locally on a single GPU.} Using these \textcolor{black}{LLMs}, OSCAR achieves 95\% to 100\% accuracy at the reported settings across the five problems, close to the observed performance of \textcolor{black}{two coding agents:} Codex with GPT-6 Astra and Claude Code with Fable~5.1. \textcolor{black}{However, the two coding agents' average token costs per run are 3.1 and 5.8 times OSCAR's, respectively.} Resampling one-shot formulations through the same Simulator can match OSCAR's accuracy, but costs 1.6 to 3.3 times as much on problems that a single attempt rarely solves. \textcolor{black}{Therefore,} the improvement loop is most useful on these harder problems.
\end{itemize}

\subsection{Managerial Insights}\label{sec:managerial}

\begin{itemize}
  \item \textbf{Evaluate available models before upgrading.} Before upgrading to frontier models for optimization modeling, firms should assess what available models can achieve through verification and iterative improvement. Cost and confidentiality both matter when formulations include sensitive business rules and data. OSCAR allows firms to use open-weight models locally or in the cloud, depending on their budget and confidentiality requirements.

  \item \textbf{Clarify business requirements with \textcolor{black}{labeled decision} examples.} Plain-language descriptions of business rules may be interpreted in different ways, leading to different judgments about whether a plan satisfies those rules. \textcolor{black}{Labeled decision examples show which decisions satisfy or violate operating rules and help make their intended meaning explicit.} Firms can make better use of optimization by maintaining these examples alongside their written descriptions. When an LLM's interpretation of the description conflicts with the examples, OSCAR follows the examples' known outcomes to build and test the Simulator.

  \item \textcolor{black}{\textbf{Keep benefiting as LLMs improve.}} OSCAR follows simple rules for allocating attempts and restarting the search. These operations are controlled by a small set of adjustable settings. As LLM capabilities and prices change, firms can tune the settings and add new models to keep benefiting from improvements in LLMs.
\end{itemize}

\subsection{Related Literature}\label{sec:lit}

\noindent\textbf{LLMs for optimization modeling.}
\textcolor{black}{Given the rapid growth of this literature, we review selected work on formulation generation, checking and iterative improvement.}
Research in this area studies both model training and the frameworks in which models operate. NL4Opt introduced a competition and benchmark for translating natural-language descriptions into optimization formulations \citep{ramamonjison2023nl4opt}. ORLM, LLMOPT, SIRL, and DPLM fine-tune models for optimization \textcolor{black}{\citep{huang2025orlm, jiang2025llmopt, chen2025sirl, zhou2026dp}}: ORLM uses synthetic modeling data, LLMOPT uses structured problem formulations, SIRL uses rewards from solver diagnostics, and DPLM addresses dynamic programming. \textcolor{black}{In our paper,} OSCAR can use specialized models as Coders for compatible problem classes, and our routing analysis allows a menu containing both specialized and general-purpose models.

\textcolor{black}{Another line of research develops agentic frameworks that coordinate LLMs for formulation generation and revision.}
Chain-of-Experts uses forward construction and backward reflection among specialized agents \citep{xiao2024coe}. OptiMUS cycles through a formulator, a programmer, and an evaluator that executes code and reports errors \citep{ahmaditeshnizi2024optimus}. LEAN-LLM-OPT adds retrieval tools for large-scale problems \citep{liang2026lean}. \textcolor{black}{In Chain-of-Experts and OptiMUS, revision is driven mainly by execution feedback. Chain-of-Experts also responds when the generated model yields no feasible solution.} \textcolor{black}{AutoREM reuses structured memory distilled from failed reformulation trajectories to improve robust-optimization reformulation \citep{chen2026autorem}.}

\textcolor{black}{Several frameworks also check formulations. LLMOPT uses an LLM to assess both the formulation and the code, then revises either when needed \citep{jiang2025llmopt}. TriVAL uses LLM validators at the semantic, formulation, and code levels \citep{fang2026trival}, and Opt-Verifier checks modeling structure and solution validity with an LLM \citep{liu2026optverifier}. OptiRepair combines solver infeasibility certificates with rationality checks \citep{ao2026optirepair}. ReLoop perturbs a formulation's parameters and tests its objective response \citep{lian2026reloop}. For vehicle-routing problems, \citet{luo2026constraintinjection} tests generated formulations with feasible and targeted constraint-violating probes to filter training data and supply rewards. The probes are labeled using a trusted Gurobi reference formulation, which OSCAR does not require.  \citet{rahman2026verisimpl} constructs queries with a solver and asks an LLM to reason about them. In constraint programming, \citet{song2026cpsync} asks an LLM to write a checker from the problem statement and finds that it often rejects correct solutions.}

 \textcolor{black}{NEMO \textcolor{black}{in \citet{song2026nemo}} and OptMaster \textcolor{black}{in \citet{lin2026optmaster}} are particularly related designs. In} \citet{song2026nemo}, a coding agent writes both an executable simulator and unit tests from the same problem description \textcolor{black}{and selected extraction}. A simulator is admitted after passing their tests. A solution is accepted when the simulator confirms feasibility and agreement with the reported objective value.  
\textcolor{black}{In \citet{lin2026optmaster},}  \textcolor{black}{OptMaster organizes attempts in a DAG, with solving, verification and summarization at each node. Like OSCAR, OptMaster uses a simulator to check feasibility and objective values. OptMaster does not focus on verifying the simulator. An LLM interprets the results to guide revisions and branch combinations, particularly when progress stalls. Passing checks need not end the search. For formulation tasks, an attempt budget and an LLM's assessment of correctness determine when the search stops.}

\textcolor{black}{Table~\ref{tab:lit} compares these frameworks, focusing on those that explicitly check formulations.} \citet{rahman2026verisimpl} and \citet{song2026cpsync} select among multiple candidate formulations. TriVAL and NEMO also consider multiple candidates within their procedures. Several frameworks report token use \citep{xiao2024coe, ahmaditeshnizi2024optimus, fang2026trival, liu2026optverifier, ao2026optirepair, lian2026reloop}. Like \citet{song2026nemo}, we use an executable simulator to evaluate solutions. \textcolor{black}{We additionally use externally supplied labeled decision examples to certify the Simulator before formulation generation, and keep the Simulator fixed during the search. Then we} 
use it to compare candidates beyond the first feasible solution and allocate costly attempts to seek further improvements.

\begin{table}[htbp]
\TABLE
{\textcolor{black}{Comparison of optimization modeling frameworks.}\label{tab:lit}}
{{\small
\setlength{\tabcolsep}{3pt}
\newcommand{\litstack}[1]{\begin{tabular}[c]{@{}c@{}}#1\end{tabular}}
\begin{tabular}{@{}lcc@{\hspace{14pt}}ccc@{\hspace{14pt}}c@{}}
\toprule
 & \multicolumn{2}{c}{Revision} & & & & \\
\cmidrule(lr){2-3}
 & \litstack{syntax\\repair} & \litstack{formulation\\revision} & \litstack{formulation\\checked by} & \textcolor{black}{\litstack{simulator\\validated by}} & \litstack{search\\behavior} & \textcolor{black}{\litstack{token\\cost}} \\
\midrule
\multicolumn{7}{@{}l}{\emph{Fine-tuned models}} \\
ORLM \textcolor{black}{\citep{huang2025orlm}}                &            &            & none      &                   &            & \\[2pt]
SIRL \citep{chen2025sirl}                 &            &            & none      &                   &            & \\[2pt]
DPLM \citep{zhou2026dp}                   &            &            & none      &                   &            & \\[2pt]
\addlinespace
\multicolumn{7}{@{}l}{\emph{Agentic frameworks}} \\
Chain-of-Experts \citep{xiao2024coe}      & \checkmark & \checkmark & none      &                   &            & reported \\[2pt]
OptiMUS \citep{ahmaditeshnizi2024optimus} & \checkmark & \checkmark & none      &                   &            & reported \\[2pt]
LEAN-LLM-OPT \citep{liang2026lean}        &            &            & none      &                   &            & \\[2pt]
\addlinespace
\multicolumn{7}{@{}l}{\emph{Frameworks that check the formulation}} \\
LLMOPT \citep{jiang2025llmopt}            & \checkmark & \checkmark & LLM & & \textcolor{black}{stops after validation} & \\[2pt]
TriVAL \citep{fang2026trival}             & \checkmark & \checkmark & LLM       &                         & \textcolor{black}{stops after validation} & reported \\[2pt]
Opt-Verifier \citep{liu2026optverifier}   & \checkmark & \checkmark & LLM       &                         & \textcolor{black}{\litstack{stops after one\\refinement pass}} & reported \\[2pt]
OptiRepair \citep{ao2026optirepair}       &            & \checkmark & \litstack{solver and\\domain checks} &                         & \textcolor{black}{stops after validation} & reported \\[2pt]
ReLoop \citep{lian2026reloop}             & \checkmark & \checkmark & solver    &                         & \textcolor{black}{stops after validation} & reported \\[2pt]
NEMO \citep{song2026nemo}                 & \checkmark & \checkmark & simulator & tests it wrote itself   & \textcolor{black}{stops after validation} & \\[2pt]
\textcolor{black}{OptMaster \citep{lin2026optmaster}} & \textcolor{black}{\checkmark} & \textcolor{black}{\checkmark} & \textcolor{black}{\litstack{simulator\\and LLM}} & \textcolor{black}{none} & \textcolor{black}{\litstack{DAG search\\guided by LLM}} & \textcolor{black}{capped} \\[2pt]
\addlinespace
OSCAR (this paper)                        & \checkmark & \checkmark & simulator & \textcolor{black}{\litstack{labeled decision\\examples given to it}} & \litstack{retains an incumbent\\and searches beyond\\feasibility} & managed \\[2pt]
\bottomrule
\end{tabular}}} 
{\textcolor{black}{Formulation checking excludes reviews during formulation generation and feedback on execution errors. ``Capped'' denotes an upper bound on the number of attempts. ``Managed'' denotes cost-aware allocation of attempts across LLMs.}}
\end{table}

\noindent\textbf{LLM routing and cascades.}
A separate literature reduces inference cost by matching queries to models. FrugalGPT learns a model cascade and confidence thresholds for accepting responses \citep{chen2023frugalgpt}. RouteLLM learns routers from preference data \citep{ong2025routellm}. AutoMix uses few-shot self-verification scores to route queries, including through a partially observable Markov decision process \citep{aggarwal2024automix}. These methods balance cost and response quality for individual queries. Our model concerns repeated attempts at the same problem under a shared latent difficulty, \textcolor{black}{allowing repeated calls to the same models.} Given success probabilities and prices, we compute a routing schedule and a competitive bound.

\noindent\textbf{Methodological foundations.} \textcolor{black}{Our analysis relates to two streams of research.  
  The first concerns Bayesian sequential experimentation, particularly within the framework of partially observable Markov decision processes (see \citealt{monahan1982state} for a survey and \citealt{alizamir2013diagnostic,araman2022diffusion} for related studies). 
Our known-prior model shares a similar updating-and-stopping structure: failed attempts update beliefs about a latent difficulty common to all calls, and we allocate repeated attempts across LLMs and determine when to stop.  The second stream applies competitive analysis to online decisions, which has been widely adopted in revenue management (e.g., \citealt{ball2009toward}) and resource allocation (e.g., \citealt{goyal2025asymptotically}). These studies evaluate policies relative to benchmarks with an informational advantage. We adopt a similar approach for the prior-free routing problem, and compare the Escalator policy's expected loss with that of a benchmark that observes the latent difficulty but not future call outcomes.}

\section{\textcolor{black}{The Simulator and Preliminary Observations}}\label{sec:env}

Before presenting OSCAR, we explain how its Simulator compares formulations, describe how we build it, and report observations from preliminary testing.

\subsection{Comparing Formulations through a Simulator}
\label{sec:simulator}

Moving beyond a one-shot formulation requires a way to compare candidates (Figure~\ref{fig:formulations}). In a fully automated framework, a human expert is unavailable for this task. An LLM arbiter faces much the same difficulty as the model that wrote the formulation: a constraint misread during drafting may also be misread during review. Further LLM reasoning can improve validation accuracy, but it adds cost and may still miss errors. \citet{fang2026trival} notes that subtle constraint violations may escape its LLM validators.

Comparing the formulations directly does not resolve this difficulty, because each evaluates objective values within its own feasible region. When two formulations yield different optimal solutions, a solution from one can be infeasible or suboptimal under the other. We therefore need a basis for comparing their solutions that is external to both formulations. Ideally, this comparison would reflect each solution's feasibility and objective value in the real system. A feasible solution is preferred to an infeasible one, and feasible solutions are ranked by objective values. We consider single-objective problems throughout, so any two feasible solutions can be compared. Since candidate solutions  \textcolor{black}{cannot}  
be tested in the real system, a Simulator evaluates them instead. It reports feasibility and, for a feasible solution, the objective value. For an infeasible solution, it is instructed to list every violated operating rule and explain each violation in plain language. Candidate solutions are evaluated and ranked by the Simulator and OSCAR retains the formulation producing the highest-ranked solution identified so far as the \emph{incumbent}.  The search continues even after a feasible solution is found,  and  stops only when the routing policy terminates the search and then returns the standing incumbent.

The Simulator is an offline oracle implemented as deterministic code. Just as we ask an LLM to write solver code rather than solve the problem itself, we ask it to write the Simulator rather than judge individual formulations. Once built, the Simulator remains fixed. Its construction requires four inputs:
\begin{itemize}
  \item A full problem description: the scenario, the goal, the supporting data and the rules that govern the system.
  \item The data of the target instance.
  \item The syntax of the operational decisions implemented in the system. Ancillary modeling variables, such as the binary indicators of a big-$M$ reformulation, are not needed.
  \item A small set of \textcolor{black}{labeled decision examples} in that syntax. Each example includes its feasibility status and, optionally, its objective value. Each infeasible example also includes a plain-language reason. Because the Simulator encodes rules rather than data, examples may come from other instances governed by the same rules. 
  \textcolor{black}{For example, in a facility planning problem, examples may be obtained from previous planning periods or similar facilities that have been built before.} 
 Optional \textcolor{black}{\emph{partial decision examples}}, each respecting or violating one constraint, help cover constraints that are easy to miss.
\end{itemize}
The first two inputs, the problem description and instance data, are standard in optimization modeling. OSCAR also requires decision syntax and \textcolor{black}{labeled decision examples}. These additional inputs record what decisions a firm implements and which plans satisfy or violate its rules. Maintaining them alongside written descriptions helps clarify plain-language rules, as discussed in Section~\ref{sec:managerial}. The syntax defines a common solution format, while the examples support testing and certification of the Simulator.

Construction proceeds iteratively. An LLM writes a candidate Simulator from the description and decision syntax, and we run it on every example. We check its feasibility classifications and, where a reference objective value is provided, its objective calculations. Failed cases and their plain-language reasons are returned to the LLM with the current code for revision. Because OSCAR uses the Simulator's explanations as well as its classifications, a separate LLM scores each rejection message on whether a reader who cannot see the code can identify the rule violated, the location of the violation and its magnitude. The mean score must reach a certain threshold. We also make examples unreadable, for instance by emptying fields, and require the Simulator to label each altered example \textsc{malformed}. Here, \emph{certification} means passing these construction checks. Certification occurs before any formulation is written (Appendix~\ref{ec:empirics}).

Two considerations make Simulator construction more tractable than optimization modeling. \textcolor{black}{First,}  \textcolor{black}{a Simulator can check a given plan by running code that simulates the system's dynamics and tests compliance with its operating rules. Writing this code is often more direct than expressing the same rules as optimization constraints, such as linear inequalities. For example, a routing Simulator can track vehicle loads and arrival times to check capacities and time windows.} \textcolor{black}{Second,} LLMs also write ordinary code in a general-purpose language more reliably than solver-specific modeling code. \citet{song2026nemo} uses this rationale and reports a similar difference in first-attempt correctness between simulators and optimization models. 

\textcolor{black}{Despite these advantages, constructing a Simulator that passes certification may require multiple attempts.} In our experiments, a weaker model produces a Simulator that passes certification for some problems. For others, it exhausts its budget and a stronger model takes over. We therefore build the Simulator using the cost-ordered escalation of Section~\ref{sec:policy}, as described in Appendix~\ref{ec:empirics}. On a few problems outside the benchmark study, Simulator construction required human help, mainly with data formatting (Appendix~\ref{ec:empirics}).

\subsection{Preliminary Observations}\label{sec:observations}

The following observations come from preliminary tests used to guide OSCAR's design, \textcolor{black}{which are} separate from the benchmark experiments in Section~\ref{sec:numerics}.

\noindent\textbf{Simulator feedback.} A Reviewer who views the problem description and formulation code, together with a binary infeasibility report, may still face difficulty in locating the formulation error. In this case, additional explanations from the Simulator can help the Reviewer diagnose these errors. We compared reviews with and without this explanation using 24 broken formulations from our own runs, all from one instance of one problem. Each formulation was reviewed twice by Qwen3.5-Flash, once with the Simulator's report and once without it. The Reviewer identified the true defect in 11 of the 24 cases without the report and in 21 with it. To examine whether diagnosis also led to repair, we tested one hard formulation in 50 review-and-repair runs of two rounds each. The share ending with a correct formulation rose from 4\% without the report to 18\% with it. These tests suggest that explanations can help both diagnosis and repair, though the evidence is limited to a small set of formulations. \citet{song2026nemo} reports a similar effect.

\noindent\textcolor{black}{\textbf{Reverting when switching to a stronger Coder.}} Repairing a weaker model's formulation requires a longer prompt that includes the flawed code and revision advice, increasing the cost per round. The flawed structure can also lead the stronger model to make superficial patches. In 50 tests, a one-shot formulation by Qwen3.6-Flash reached the optimum 23 times, compared with 19 times when Qwen3.5-Flash wrote the initial formulation and Qwen3.6-Flash refined it. The latter approach also used more tokens. \textcolor{black}{This comparison motivates discarding a weaker Coder's unsuccessful revisions before a stronger Coder takes over.} The stronger Coder returns to the best formulation retained so far.

\noindent\textbf{Limits of repeated revision.} Runs that reached the optimum typically did so within three to five revisions. Those that had not done so by then seldom succeeded later, even when allowed up to thirty revisions. Late revisions tended to make cosmetic changes, such as adding comments or renaming variables, and sometimes degraded the code. For example, splitting a constraint into equivalent pieces could slow the solver on large instances. \textcolor{black}{These observations motivate periodically discarding revisions and returning to the best formulation retained so far after repeated revisions fail to yield improvement (Section~\ref{sec:policy}).}

\noindent\textbf{Structure construction and syntax repair.} Organizing a description into a structured form before generating code is common practice and is the role of the structure constructor in Section~\ref{sec:framework}. Independent calls to the structure constructor may end up with slightly different structures, and affect the final formulation's quality to some extent, but an imperfect structure does not preclude a correct formulation after subsequent revision.  
\textcolor{black}{Syntax errors are common in generated code.} Feeding error messages back to the cheapest model on the menu resolved nearly all such errors in our tests, usually within two rounds.

\noindent\textbf{Cost of coding and reviewing.} The Coder and Reviewer read prompts of similar length, but the Coder writes a complete program while the Reviewer writes a short diagnosis. We tracked token usage and cost during preliminary testing. A Reviewer call averaged 5,528 \textcolor{black}{input tokens} and 6,077 \textcolor{black}{output tokens} at a cost of \$0.007. A Coder call averaged 6,696 \textcolor{black}{input tokens} and 11,184 \textcolor{black}{output tokens} at a cost of \$0.013, about 1.8 times as much.  \textcolor{black}{The output counts include thinking tokens.} The choice of models for the two roles therefore affects \textcolor{black}{the cost of a revision attempt}, with an increase in the Coder's tier costing more than an increase in the Reviewer's tier.

\section{The OSCAR Framework}\label{sec:framework}

OSCAR has three main components: the Simulator \textcolor{black}{as introduced in} Section~\ref{sec:simulator}, a Coder and a Reviewer. The Coder and Reviewer are selected from a menu of LLMs, which may be general-purpose models or models fine-tuned for optimization modeling. \textcolor{black}{A model's position on this menu, ordered by price, is its \emph{tier}. Higher-tier models are usually stronger.} We use \emph{tier} when the position itself matters, as in the adjacent-tier restriction of Section~\ref{sec:model}. Figure~\ref{fig:OSCAR} summarizes the preparations built once, the improvement loop and the outcome.

\begin{figure}[htbp]
\FIGURE
{\includegraphics[width=0.9\textwidth]{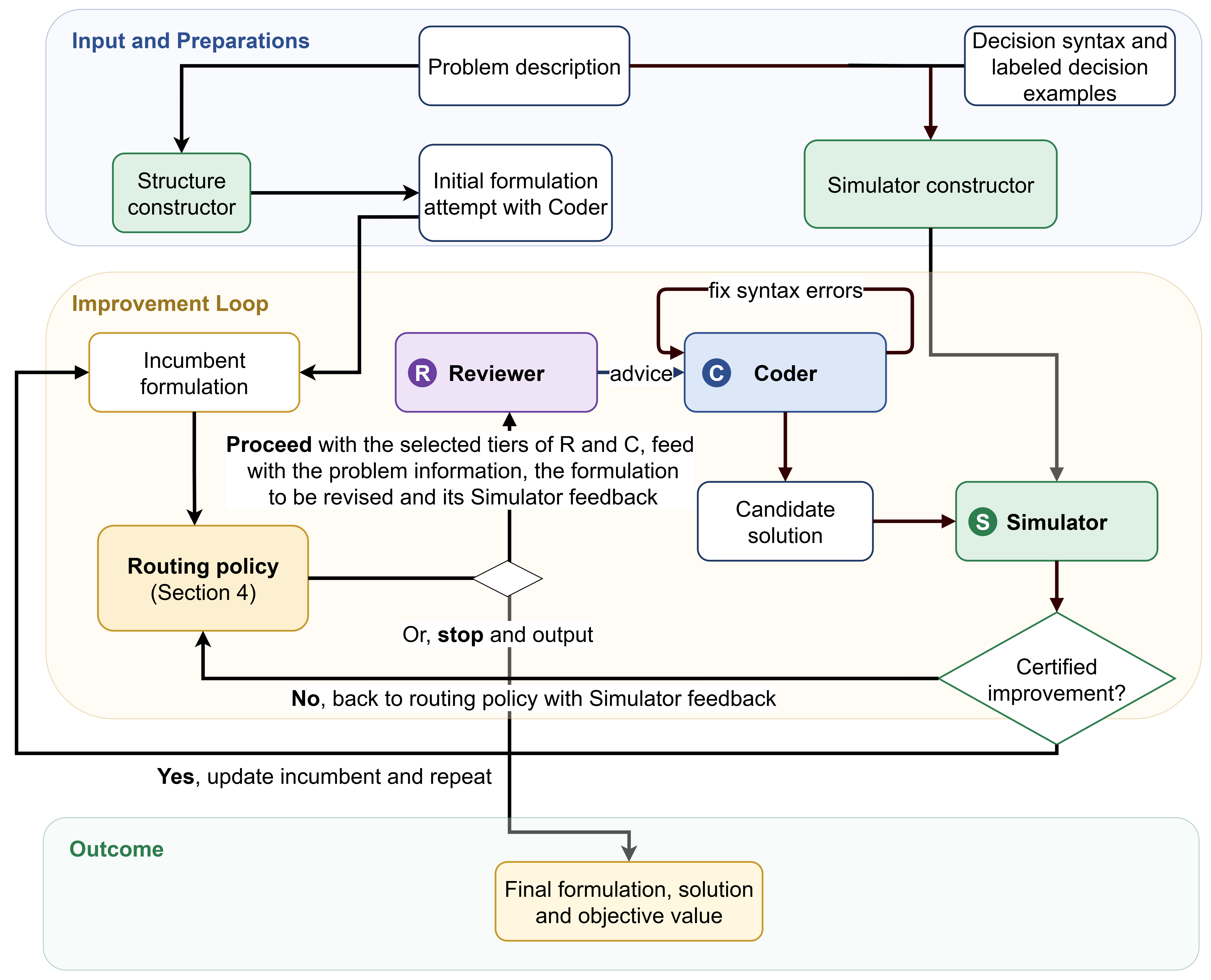}}
{The OSCAR framework, with some implementation details omitted.\label{fig:OSCAR}}
{}
\end{figure}

\noindent\textbf{Input and preparations.}
OSCAR begins with the four inputs of Section~\ref{sec:simulator}. The structure constructor converts the plain-language description into a schema that lists every required parameter, specifying \textcolor{black}{its data type, its dimensions (for example, one entry per customer),} and a one-line description. Later calls receive both the schema and the original description. This step follows established practice \citep{ahmaditeshnizi2024optimus, liang2026lean}. It uses a single call to the cheapest model on the menu and costs a fraction of a cent per run (Table~\ref{tab:num-cost}). After the Simulator constructor builds and certifies the Simulator, a first Coder call produces the initial formulation in one shot. \textcolor{black}{We use a strong model for the initial formulation. \textcolor{black}{When the Coder switches to a stronger LLM, it discards unsuccessful revisions and resumes from the incumbent} (Section~\ref{sec:observations}). Starting with a weak model would leave it revising a weak model's output even after this return.} The initial formulation becomes the incumbent for the improvement loop, with no guarantee of correctness.

\noindent\textbf{The improvement loop.}
At the start of each pass, the routing policy either stops or selects the Reviewer and Coder tiers for the next attempt. We take this policy as given here and formalize its decisions in Section~\ref{sec:model}. The Reviewer receives the problem information, the formulation to be revised and the Simulator's feedback on that formulation, then returns revision advice. Within a window, this formulation is either the latest draft or the incumbent retained at the start of the window (Section~\ref{sec:policy}). The Coder uses the advice to produce a candidate formulation and its solver code.

If the code fails to run, we repair it in a short loop, as described in Section~\ref{sec:observations}. The Simulator then evaluates the resulting solution against the operating rules. An attempt that produces no candidate solution bypasses the Simulator. \textcolor{black}{OSCAR reports this outcome as a possible sign that the formulation is over-constrained.}

\noindent\textbf{Certification and termination.} The Simulator's report ranks each candidate solution against the incumbent's solution. Every feasible solution ranks above every infeasible one, and feasible solutions are compared using objective values recomputed by the Simulator. Among infeasible solutions, the one that violates fewer operating rules ranks higher. A \emph{malformed} solution is one the Simulator cannot read. It is treated as a syntax error and returned to the repair loop, so it is never ranked.

\textcolor{black}{OSCAR certifies an improvement when a candidate ranks strictly above the incumbent.} That candidate replaces the incumbent and the loop continues. A candidate that does not improve the ranking leaves the incumbent unchanged, although the next attempt may continue to revise that candidate, subject to the reverting rule. The session ends when the routing policy stops or an external wall-clock or spending ceiling is reached (Section~\ref{sec:policy}). OSCAR then returns the standing incumbent with its solution and objective value.

\noindent\textbf{Routing decisions.}
\textcolor{black}{The routing policy determines which LLMs OSCAR calls as Reviewer and Coder, and when it stops.} \textcolor{black}{Using stronger LLMs may increase the chance of improving the incumbent, but also raises the cost per attempt.} The \textcolor{black}{routing policy} must choose between them with limited knowledge of the difficulty of the current improvement and decide \textcolor{black}{whether to continue after each failed attempt}. As Section~\ref{sec:model} shows, the stopping decision also depends on whether the incumbent is feasible. That section formalizes these decisions and derives the Escalator policy.

\section{Routing Control: The Escalator Policy}\label{sec:model}

\textcolor{black}{In the framework introduced in Section~\ref{sec:framework}, we have not yet specified which LLMs to call or when to stop. We now formulate these routing and stopping decisions and analyze the Escalator policy.} We analyze one improvement window at a time. A certified improvement replaces the incumbent and opens the next window. \textcolor{black}{We apply the same decision model to each window.} It is therefore enough to study an arbitrary window.

\subsection{General Formulation of Sequential Improvement Search}
\label{sec:protocol}
 \textcolor{black}{The overall improvement trajectory of an OSCAR session is hard to model directly for two reasons. First, a formulation contains interacting variables, constraints and objective terms. Problems also differ in the number and types of constraints they may require. Each revision may add, remove or modify several elements at once, correcting some errors while introducing others. These differences make it difficult to construct a common model of how formulations evolve and how revisions affect the resulting plans and the prospects for further improvement.}

\textcolor{black}{Second, progress can take different forms. Before feasibility is reached, a revision may reduce the number of operating rules violated by the resulting plan. Afterward, progress is measured by the objective value recomputed by the Simulator. We simplify the analysis by studying the search for the next certified improvement, without modeling the full sequence of formulation changes or the magnitude of each gain.} We define an \emph{improvement window} as the segment between two consecutive certified improvements. An improvement is certified when the Simulator ranks the new solution above the incumbent according to the ordering defined in Section~\ref{sec:framework}. Thus, reducing the number of violated rules counts as a certified improvement, just as improving the objective value does.

Formally, each window begins with a fixed incumbent, and there is no limit on the number of LLM calls that can be made within that window. A window terminates when either the Simulator certifies an improvement or the controller chooses to stop. If an improvement is certified, the improved formulation becomes the incumbent for the next window; otherwise, the incumbent remains unchanged.
\textcolor{black}{In the single-window model, every certified improvement yields the same reward $G>0$.}
 The parameter $G$ reflects how we weigh the risk of ultimately missing a correct formulation against the token cost.

We distinguish two types of improvement windows according to the feasibility of the incumbent at the start of the window. A \emph{Type~1} window begins with an infeasible incumbent. \textcolor{black}{If the underlying problem admits a feasible solution, that solution improves on the incumbent.}  
\textcolor{black}{A \emph{Type~2} window begins with a feasible incumbent, for which further improvement may not be possible.}

We make three simplifications to study a single improvement window. First, we take the total number of windows needed to reach the best formulation to be unaffected by within-window operations. This assumption is what permits treating $G$ as a constant reward per window. Second, we model the difficulty of reaching a certified improvement as fixed within a window. In particular, our analysis does not distinguish whether a revision starts from the incumbent or from a formulation already revised within the window. Section~\ref{sec:policy} sets the implementation choice between continuing to revise the current formulation and reverting to the incumbent. Third, we defer the initial one-shot coding round to Section~\ref{sec:policy} and Appendix~\ref{ec:num-models}.

\vspace{1mm}
\noindent
\textbf{Model primitives and dynamics.}
At the beginning of each improvement window, we assume there is a hidden state $\theta\in\Theta$ reflecting the difficulty of improving upon the incumbent. This state is fixed but remains unobservable to the controller throughout the window. Let $\mathcal{D}$ denote   a discrete set of improvable states, and let $o$ denote the    non-improvable state. For Type~1 window, we assume  $\Theta:=\mathcal D$, whereas for Type~2 window, we assume $\Theta:=\mathcal D\cup\{o\}$. 
Let $\mathcal A=\{1,\ldots,m\}$ denote the set of available \emph{configurations}. \textcolor{black}{A configuration is an ordered (Reviewer, Coder) pair of LLMs used in one revision attempt, and the two roles may be filled by models of different tiers.} We restrict the menu to configurations whose Reviewer and Coder are of the \emph{same or adjacent} tiers. The restriction is both reasonable and convenient.

\begin{itemize}
\item \textbf{Little to gain from distant pairs.} A configuration's cost is dominated by its more expensive component. Pairing a strong LLM with a much weaker counterpart therefore saves little. Under the policy we propose in Section~\ref{sec:competitive}, a strong model is reached only after weaker configurations have already failed, {\color{black}which signals a challenging window and calls for a capable counterpart}.
\item \textbf{Less to calibrate.} With $L$ LLM tiers, the restriction \textcolor{black}{to pairs of the same or adjacent tiers} gives $3L-2$ configurations instead of $L^{2}$, which is easier to tune in practice. It also reduces what must be measured before the policy can be run. With an unrestricted menu, configurations would first have to be ranked by cost and screened for dominance in success probability, which takes measurement. \textcolor{black}{With this restriction, the cost ranking is known in advance. Raising the Coder's tier costs more than raising the Reviewer's tier by the same amount. We therefore order configurations first by the Coder's tier and then by the Reviewer's.}
\end{itemize}

Each configuration is treated as a single action with one combined price and one certified-improvement probability.
At each decision epoch, the controller selects an action from $\mathcal U:=\mathcal A\cup\{\text{stop}\}.$
Running configuration $a\in\mathcal A$ incurs cost $c_a>0$. We index the configurations in nondecreasing order of cost, so that $c_1\le c_2\le\cdots\le c_m.$
Conditional on state $\theta\in\mathcal D$, a run of configuration $a$ produces a certified improvement with probability $q_{a,\theta}\in(0,1)$ and otherwise fails. In state $o$, no improvement is available, so $q_{a,o}=0$ for every $a\in\mathcal A$. Conditional on $\theta$, we assume call outcomes are independent across decision epochs. We further assume that an unsuccessful candidate has no residual value in the model. A successful call generates reward $G$ and terminates the window, whereas a failed call generates no reward and leads to another routing decision. Choosing $\text{stop}$ terminates the window without reward.

\vspace{1mm}
\noindent 
\textbf{Deterministic routing plans.} 
Let \(\Pi\) denote the set of deterministic routing plans. Because a successful call terminates the window and all subsequent decisions are made only after failures, a routing plan is fully characterized by its actions along the all-failure path.  We therefore write a routing plan as
$$ \pi=(a_1,\ldots,a_N;\mathrm{stop}), \qquad a_i\in\mathcal A, \qquad N\in\{0,1,\ldots\}\cup\{\infty\}. $$
Each scheduled call $a_i$ is made if and only if all preceding calls $a_1,a_2,\ldots, a_{i-1}$ have failed, and a certified improvement terminates the plan immediately. For finite \(N\), the plan stops if all \(N\) calls fail, with \(N=0\) representing immediate stopping. When \(N=\infty\),   calls are made without a terminal \(\mathrm{stop}\), so the plan continues after every failure.

The analysis in Sections~\ref{sec:prior} and~\ref{sec:competitive} will focus on a special class of routing policies, which we call Escalator policies.
An \emph{Escalator policy} follows a routing plan whose configuration indices are non-decreasing along the all-failure path.
Within an improvement window, it may repeat or skip configurations but never moves backward through the cost-ordered menu.

\vspace{1mm}
\noindent 
\textbf{Statewise expected reward.}
For a plan $\pi\in\Pi$ and state $\theta\in\Theta$, define the survival probability immediately before call $i$ as
 $S_{i,\theta}:=\prod_{k<i}(1-q_{a_k,\theta}).$
Thus, $S_{i,\theta}$ is the probability that call $i$ is reached after all previous calls have failed.
For a finite plan with $N<\infty$, define $S_{N+1,\theta}:=\prod_{k=1}^{N}(1-q_{a_k,\theta}),$
which is the probability that the plan terminates without an improvement.
The expected net reward conditional on $\theta$ is
\begin{equation}
\label{eq:statewise-value}
V(\pi\mid\theta)
=
G(1-S_{N+1,\theta})
-
\sum_{i=1}^{N}c_{a_i}S_{i,\theta},
\qquad \theta\in\mathcal D .
\end{equation}
For an infinite plan, the same expression is interpreted by taking the limit as $N\rightarrow\infty$, namely
\[
V(\pi\mid\theta)
=
\liminf_{N \to \infty}\ \left\{ G\left(1-S_{N+1,\theta}\right)
-
\sum_{i=1}^{N}c_{a_i}S_{i,\theta}\right\}.
\]
For every $\theta\in\mathcal D$,  $V(\pi|\theta)$ in the case of $N=\infty$ is well defined because the finite configuration set and strictly positive success probabilities imply
\[
\sum_{i=1}^{\infty}c_{a_i}S_{i,\theta}
\le
\frac{c_m}{\min_{a\in\mathcal A}q_{a,\theta}}
<\infty .
\]
For the non-improvable state $\theta=o$, all calls fail with probability one, so any infinite plan incurs infinite expected cost and $V(\pi\mid o)=-\infty$.

At this stage, we have not imposed a probability distribution on the hidden state $\theta$ or otherwise specified how statewise expected reward should be aggregated. Consequently, the general formulation does not yet define an optimal routing plan. In 
Section~\ref{sec:prior}, we assume a known prior distribution over $\theta$ and formulate the resulting Bayesian routing problem as a Partially Observable Markov Decision Process (POMDP). In Section~\ref{sec:competitive}, we treat $\theta$ as an unknown fixed parameter, and evaluate an Escalator policy against a clairvoyant who knows $\theta$, through competitive analysis.

\subsection{Exact Optimality in the Two-by-Two Case}\label{sec:prior}
In this subsection, we specialize the general sequential improvement search in Section~\ref{sec:protocol} to two configurations and a known prior over two improvable states, with an additional non-improvable state in Type~2 windows. We formulate the problem as an infinite-horizon POMDP and establish the Escalator structure of an optimal routing plan.

Let $\mathcal A=\{A,B\}$, where $A$ is the lower-cost \textcolor{black}{configuration} and $B$ is the higher-cost \textcolor{black}{configuration},
with $0<c_A<c_B$.  For Type~1 window, the hidden state is $\theta\in\Theta:=\{e,h\},$ and for Type~2 window, the hidden state is $\theta\in\Theta:=\{e,h,o\},$ 
where $e$ and $h$ indicate that the next improvement is, respectively, easy or hard to obtain,
and $o$ indicates that no further improvement over the opening incumbent is available. At the
beginning of the window, the controller has a known prior $ \mu=(\mu_\theta)_{\theta\in\Theta}\in\mathbf{\Delta}(\Theta):=\{\mathbf x: \sum_{\theta \in \Theta}x_{\theta}=1, x_{\theta}>0 \text{ for all } \theta \in \Theta\}$. As in Section~\ref{sec:protocol}, a call to \textcolor{black}{configuration} $a\in\{A,B\}$ succeeds with probability
$q_{a,\theta}\in(0,1)$ when $\theta\in\{e,h\}$, while $q_{a,o}=0$.  For Type~1 windows, we set $\mu_o=0$ and $\beta_o=0$ for every posterior belief $\beta$ whenever these coordinates appear below.

\vspace{1mm}
\noindent\textbf{Updating rule of posterior belief.}
Let $\beta=(\beta_\theta)_{\theta \in \Theta}$ denote the posterior belief immediately before a decision.
The posterior success probability of \textcolor{black}{configuration} $a$ is
\begin{equation}
  p_a(\beta):=\sum_{\theta\in\Theta}\beta_\theta q_{a,\theta}
  =\beta_e q_{a,e}+\beta_h q_{a,h}.
  \label{eq:posterior-success}
\end{equation}
If the call to \textcolor{black}{configuration} $a$ fails, let $F_a$ denote the failure event and the current belief is updated from $\beta$ to $T_a(\beta)$ according to the Bayes' rule as follows: 
\begin{equation}
  T_a(\beta)_\theta
  :=\frac{\Pr(F_a\mid \theta, a)\beta_{\theta}}{ \sum_{\theta' \in  \Theta} \Pr(F_a\mid \theta', a) \beta_{\theta'}}=\frac{\beta_\theta(1-q_{a,\theta})}{1-p_a(\beta)},
  \qquad \theta\in\Theta.
  \label{eq:failure-update}
\end{equation}
No posterior following success is needed because success terminates the window.
Therefore, after
$n_A$ failed calls to $A$ and $n_B$ failed calls to $B$, the posterior belief can be expressed as: 
\begin{equation}
  \beta_\theta(n_A,n_B)
  =\frac{\mu_\theta(1-q_{A,\theta})^{n_A}(1-q_{B,\theta})^{n_B}}
{\displaystyle\sum_{\theta'\in\Theta}
   \mu_{\theta'}(1-q_{A,\theta'})^{n_A}(1-q_{B,\theta'})^{n_B}},
  \qquad \theta\in\Theta.
  \label{eq:count-posterior}
\end{equation}

 \vspace{1mm}
 \noindent 
 \textbf{Optimal value function and Bellman equation.}
Let $J_{\infty}(\beta)$ denote the optimal expected net
reward starting from belief $\beta$, also referred to as the optimal value function:
\[
J_\infty(\beta)
:=\sup_{\pi\in\Pi}  V_{\beta} (\pi) = 
\sup_{\pi\in\Pi}
\sum_{\theta\in\Theta}
\beta_\theta V(\pi\mid\theta)\textcolor{black}{.}
\]  
To compute $J_\infty(\beta)$,   consider a finite-horizon problem with at most $\bar N  < \infty$  calls. For each $k=0,1,\ldots, \bar N$, 
let $J_k(\beta)$ denote the optimal expected net reward at belief $\beta$ when there are $k$ available calls to make. For any value function $f$, \textcolor{black}{define the Bellman operator:}
\begin{align}
    (\mathcal{L}_a f)(\beta)
  :=-c_a+G p_a(\beta)
   +\bigl(1-p_a(\beta)\bigr)f \bigl(T_a(\beta)\bigr),
  \qquad a\in\{A,B\}.  \label{eq-one step-Bellman operator}
\end{align}
Then the functions $\{J_k(\beta): k=0,1,\ldots, \bar N\} $ satisfy the following finite-horizon Bellman equation: $J_0(\beta)=0$ for any reachable belief $\beta$, and  for $k=1,2,\ldots, \bar N$ and any reachable belief $\beta$, 
\begin{align}
     J_{k}(\beta) 
  =\max\bigl\{0, \ (\mathcal{L}_A J_{k-1})(\beta),\ (\mathcal{L}_B J_{k-1})(\beta) \bigr\}. 
  \label{eq:belief-bellman}
\end{align} 
The following lemma shows the convergence  of $\{J_k(\beta): k=0,1,\ldots, \bar N\} $   to $J_{\infty}(\beta)$, whose proof is deferred to Appendix~\ref{appendix-proof-lemma-Bellman}. 
\begin{lemma}\label{lemma-sec 4.2-Bellman infinite horizon}
    For  any reachable belief $\beta$, 
    $\lim_{k \to \infty} J_k(\beta) = J_{\infty}(\beta)$, and  $J_{\infty}(\beta)$ satisfies the following infinite-horizon Bellman equation:  $J_{\infty}(\beta) = \max\{0,\ (\mathcal{L}_A J_{\infty})(\beta),\ (\mathcal{L}_B J_{\infty})(\beta) \}. $ 
\end{lemma}

\vspace{1mm}
\noindent
\textbf{Structural conditions.}
For each \textcolor{black}{configuration} $a\in\{A,B\}$ and improvable state
$\theta\in\{e,h\}$, define the cost-effectiveness ratio: $\gamma_{a,\theta}:={c_a}/{q_{a,\theta}}.$
Given state $\theta$ and \textcolor{black}{configuration} $a$, 
$\gamma_{a,\theta}$ is the expected total calling cost required to
obtain an improvement. Thus, a smaller value of $\gamma_{a,\theta}$ indicates
greater cost efficiency. We impose the following structural conditions
throughout this subsection.
\begin{assumption}
\label{ass:two-model-structure}
\begin{enumerate}
  \item[(i)] For each $a\in\{A,B\}$, $q_{a,e}>q_{a,h}$; 
  \item[(ii)] For each $\theta\in\{e,h\}$, $q_{B,\theta}\ge q_{A,\theta}$; 
  \item[(iii)] The two inequalities $ \gamma_{B,e}<\gamma_{A,e}$ and $ \gamma_{B,h}>\gamma_{A,h}$ 
  do not hold simultaneously.
  \item[(iv)]  $G\ge\gamma_{B,h}$. 
\end{enumerate}
\end{assumption}

Condition~(i) orders the two improvable states by difficulty: every \textcolor{black}{configuration} has
a higher success probability in state $e$ than in state $h$.  
Condition~(ii)
identifies $B$ as the more capable \textcolor{black}{configuration}, in the sense that it weakly dominates
$A$ in success probability in both states. Condition~(iii) imposes a
single-crossing restriction on the \textcolor{black}{configurations'} cost-per-success rankings. Because
a lower $\gamma_{a,\theta}$ indicates greater cost efficiency, it rules out
the reverse crossing in which $B$ is more cost-efficient in the easy state
while $A$ is more cost-efficient in the hard state. Thus, if the ranking
crosses, the permitted direction is that $A$ is more cost-efficient in state
$e$ and $B$ is more cost-efficient in state $h$. This directional crossing
provides the economic basis for one-way escalation from $A$ to $B$.
Condition~(iv) ensures that the reward is at least  the expected total cost of calling $B$ until success, even in the hard state. \textcolor{black}{It is used} in Theorem~\ref{thm:escalator}(i) to show that stopping is never strictly optimal in Type~1 windows.

\vspace{2mm}
We first present a lemma that compares the two possible orderings of an $A$-call and a $B$-call and characterizes how the preferred ordering evolves as failures update the posterior belief. 
\begin{lemma}[Adjacent interchange]
\label{lem:bellman-interchange}
Fix a reachable posterior belief $\beta$ and a bounded
 value function $f$.
Consider the two call orders $AB$ and $BA$, where the second call
is made only if the first fails, and both orders use the same
  value function $f$ if both calls fail. Then 
\begin{align}
  (\mathcal{L}_A\mathcal{L}_B f)(\beta)
-
(\mathcal{L}_B\mathcal{L}_A f)(\beta)
=
c_Bp_A(\beta)-c_Ap_B(\beta)= \sum_{\theta\in\{e,h\}}
\beta_\theta q_{A,\theta}q_{B,\theta}
\bigl(\gamma_{B,\theta}-\gamma_{A,\theta}\bigr). \label{eq-AB-BA}
\end{align} 
Moreover, the sign of the difference in \eqref{eq-AB-BA} is the sign of
\begin{align}
  D(\lambda)
:=
q_{A,e}q_{B,e}
  \bigl(\gamma_{B,e}-\gamma_{A,e}\bigr)
+
q_{A,h}q_{B,h}
  \bigl(\gamma_{B,h}-\gamma_{A,h}\bigr) \cdot \lambda\textcolor{black}{,} 
\label{eq-difference in lambda} 
\end{align}
where $\lambda=\lambda(\beta):=\beta_h/\beta_e$. Under parts (i) and (iii)
of Assumption~\ref{ass:two-model-structure}, $\lambda$ strictly increases along any all-failure path,
and $D(\lambda)$ can change sign at most once, and only from nonnegative
to nonpositive. 
\end{lemma}

 Lemma~\ref{lem:bellman-interchange} shows that
\(AB\) is weakly preferred to \(BA\) if and only if ${c_A}/{p_A(\beta)}
\leq
 {c_B}/{p_B(\beta)}.$
Thus, the \textcolor{black}{configuration} with the lower cost per success under the current posterior
should be placed first. 
The single-crossing result describes how this pairwise ordering evolves
after failures. Under
Assumption~\ref{ass:two-model-structure}(i), a failed call to \textcolor{black}{configuration} \(a\)
multiplies the posterior probability of the hard state relative to the easy state
by $(1-q_{a,h})/(1-q_{a,e})>1$. 
Hence, \(\lambda\) increases along the all-failure path. In
\(D(\lambda)\), the constant term reflects the \textcolor{black}{configurations'} cost-effectiveness
ranking in the easy state, whereas the coefficient of \(\lambda\) reflects
their ranking in the hard state. A change in \(D(\lambda)\) from negative
to positive would require \(B\) to be more cost-effective in the easy state
and \(A\) to be more cost-effective in the hard state. Assumption
\ref{ass:two-model-structure}(iii) rules out precisely this reverse
crossing. Therefore, as failures accumulate, the pairwise advantage of
\(AB\) over \(BA\) can change sign at most once, and only from non-negative
to non-positive. The proof of Lemma~\ref{lem:bellman-interchange} is deferred to Appendix~\ref{appendix-lemma-adjacent exchange}.

 \vspace{2mm}
The following lemma provides sufficient conditions under which replacing a routing plan's terminal call to  $A$ with  $B$ does not decrease its expected value. 
\begin{lemma}[Terminal replacement]
\label{lem:terminal-replacement}
Suppose Assumption~\ref{ass:two-model-structure}(ii) holds. Consider a finite routing plan whose last call is to \(A\), reached at posterior belief \(\beta\).  If
\[
  Gp_A(\beta)\ge c_A
  \qquad\text{and}\qquad
  \frac{c_B}{p_B(\beta)}\le\frac{c_A}{p_A(\beta)},
\]
then replacing the last call by $B$ does not decrease the plan's value.
\end{lemma}

The intuition behind Lemma~\ref{lem:terminal-replacement} is that the lower-cost \textcolor{black}{configuration} $A$ may be useful
early in the search as an inexpensive choice, but this does not apply to the final call because no future decision can benefit from the
information generated by a failure. The first condition ensures that
continuing with $A$ is still worthwhile relative to stopping, whereas the
second identifies a posterior at which $B$'s greater capability is
available at no higher cost per success. In this region, the higher price of
$B$ is justified by its capability advantage under
Assumption~\ref{ass:two-model-structure}(ii). Hence, an optimal plan has no
reason to end with $A$ and replacing the terminal $A$-call by $B$ cannot
reduce its value. The proof of Lemma~\ref{lem:terminal-replacement}
 is provided in Appendix~\ref{appendix-lemma-terminal replacement}.

\vspace{2mm}
We now  state the main result in this subsection. 
\begin{theorem}[Optimal routing within a window]\label{thm:escalator}
  Suppose Assumption \ref{ass:two-model-structure}(i)--(iii) hold. 
For any reachable belief $\beta$, $J_\infty(\beta)$ is attained
by a deterministic routing plan of the form
\begin{align} 
\bigl(
\underbrace{A,\ldots,A}_{n_A\text{ calls}},
\underbrace{B,\ldots,B}_{n_B\text{ calls}};
\mathrm{stop}
\bigr),\label{eq:escalator-form}
\end{align}
where $n_A,n_B\in\mathbb{N} \cup\{\infty\}$,
with $n_B=0$ if $n_A=\infty$. 
Moreover, the following results hold: 
\begin{itemize}
    \item[(i)] For Type~1 window, under Assumption~\ref{ass:two-model-structure}(iv),  
    stopping before success is never strictly optimal. In other words, $n_A=\infty$ or $n_B=\infty$.  
  If, additionally, $\gamma_{B,h}<\gamma_{A,h}$, then  $n_A $ can be chosen to be finite;

\item[(ii)] For Type~2 window,
every optimal deterministic
routing plan starting from prior  $\mu$ stops after at most
$\lfloor{G(1-\mu_o)}/{(\mu_o c_A)} \rfloor$ consecutive failures.  
\end{itemize}  
\end{theorem}

The structure in \eqref{eq:escalator-form} follows by combining the local comparison results 
in Lemmas~\ref{lem:bellman-interchange} and~\ref{lem:terminal-replacement} with a finite-truncation argument \textcolor{black}{using the convergence established in Lemma~\ref{lemma-sec 4.2-Bellman infinite horizon}}.
Specifically, for each truncated problem with finitely many calls permitted, the single-crossing property in
Lemma~\ref{lem:bellman-interchange}  allows an optimal  sequence to be arranged
as $A^rB^sA^u$. If the terminal $A$-block   follows a non-empty
$B$-block, the final $A$-call must be worthwhile relative to
stopping, and $B$ is no less cost-effective at that belief.
Lemma~\ref{lem:terminal-replacement} therefore permits replacing this call by $B$, while
Lemma~\ref{lem:bellman-interchange}  allows the replacement call to be moved leftward
and merged with the preceding $B$-block. Repeating these
operations eliminates the terminal $A$-block.
The convergence in Lemma~\ref{lemma-sec 4.2-Bellman infinite horizon},  together with a subsequence
construction, then yields an Escalator plan attaining
the infinite-horizon optimal value.  More details can be found from Appendix~\ref{appendix-thm-escalator}.    

Parts (i) and (ii)   further explain the different stopping behavior in the two window types. In a Type~1 window, the incumbent is infeasible and an improvement is available. Under \textcolor{black}{Assumption~\ref{ass:two-model-structure}(iv)}, calling $B$ until success has non-negative expected net reward, so an optimal controller can continue until a certified improvement is obtained. In a Type~2 window, the incumbent is feasible, and repeated failures increase the posterior probability that no improvement is available. 
After sufficiently many failures, the expected reward from any continuation falls below even its calling cost, making stopping strictly optimal. Both proofs are given in Appendix~\ref{appendix-corollary-nostop}.

\vspace{2mm}
In the following corollary, we characterize the condition under which both blocks for $A$ and $B$ in the optimal Escalator policy stated in Theorem~\ref{thm:escalator} can be chosen to be nonempty. 
 \begin{corollary}[Non-empty calls to both \textcolor{black}{configurations}]\label{corollary-positive calls}
Suppose  Assumption \ref{ass:two-model-structure}(i)--(iii) hold and  $Gp_A(\mu)-c_A
\ge
\max\{0,\;G(1-\mu_o)-c_B\}.$
Then the following statements hold:
\begin{enumerate}
\item[(i)]
For Type~1 window, if Assumption~\ref{ass:two-model-structure}(iv)
holds and $\gamma_{B,h}<\gamma_{A,h}$, an optimal
Escalator plan can be chosen with
$1\le n_A<\infty$ and $n_B=\infty$; 
\item[(ii)]
For Type~2 window, define $n_0
:=
\min\left\{
n\in\mathbb N^+:
Gp_A(T_A^n(\mu))\le c_A
\right\},$
which is finite.
If $Gp_B(T_A^{n_0}(\mu))-c_B\ge0,$
then an optimal Escalator plan can be chosen with
$1\le n_A, n_B<\infty$. 
\end{enumerate} 
\end{corollary} 

At the initial belief $\mu$, the common condition ensures that
one call to $A$ followed by stopping weakly dominates both
immediate stopping and every $B$-only plan. In other words, an optimal plan
can be chosen to begin with $A$. For a Type~1 window, under
the additional conditions in part~(i), Theorem~\ref{thm:escalator}(i) ensures
that an optimal plan 
can be chosen with a finite $A$-block followed by an infinite
$B$-block. For a Type~2 window, 
we first prove that 
$n_0$ is an optimal stopping
point when only $A$ is used. In fact, along the all-$A$ failure path, the expected net reward
of one additional $A$-call followed by stopping if it fails, 
$Gp_A(T_A^n(\mu))-c_A$, is non-increasing in $n$
and is non-positive at $n=n_0$. This ensures the optimality of stopping after $n_0$ failures  if only $A$ is used.  
Then the condition
in part~(ii) ensures that a $B$-call still has non-negative
expected net reward at this stopping point. Thus, if an
$A$-only plan is optimal, appending one $B$-call to
$(A^{n_0};\mathrm{stop})$ preserves optimality.  This ensures an optimal Escalator plan has both blocks non-empty. 
The proof is provided in Appendix~\ref{appendix-corollary-escalator}.

\subsection{A  Competitive Guarantee in the General Case}\label{sec:competitive}

 As shown in Subsection~\ref{sec:prior}, with two
configurations, two improvable states, and a known prior, an optimal policy has an Escalator
structure. This motivates us to investigate whether a general low-to-high routing
principle remains effective for an arbitrary finite menu of configurations.
 
\textcolor{black}{We study the general model of Section~\ref{sec:protocol}.}
In practice, the distribution of task difficulty may be unavailable,
unstable, or difficult to estimate reliably. We therefore do not assume a
prior distribution over the hidden state and instead seek a deterministic,
prior-free Escalator policy that provides a uniform competitive guarantee
against a state-aware clairvoyant benchmark.  
 We first analyze Type~2 windows, for which $\Theta=\mathcal D\cup\{o\}$, and return to Type~1 windows at the end of this subsection.

\vspace{1mm}
\noindent
\textbf{Clairvoyant benchmark and competitive ratio.} To evaluate a prior-free routing policy, we compare it with a state-wise
clairvoyant controller that observes the realized state $\theta$ before making
any calls.  Define its optimal expected net reward by $V^{\rm CL}(\theta):=\sup_{\pi\in\Pi}V(\pi \mid\theta).$
For $\theta\in\mathcal D$, define
$\gamma_{a,\theta}:=c_a/q_{a,\theta}$ as in Section~\ref{sec:prior}.  This quantity is the expected total
cost of repeatedly calling configuration $a$ until success.  Once
$\theta$ is known, a failure provides no additional information and leaves the
controller facing the same decision problem. Therefore, the clairvoyant follows the same
model $ \arg\min_{a \in \mathcal{A}} \gamma_{a,\theta}$  until success when $\min_{a \in \mathcal{A}}\gamma_{a,\theta}<G$, and stops
immediately otherwise, yielding
\begin{align}
    V^{\rm CL}(\theta)
=
\left(G-\min_{a\in\mathcal A}\gamma_{a,\theta}\right)^+,
\qquad \theta\in\mathcal D, \label{eq-CL-reward}
\end{align}
while $V^{\rm CL}(o)=0$ for the non-improvable state.  We measure the
performance of a given routing policy $\pi$ by its worst-case
competitive ratio in expected loss
\[
\operatorname{CR}(\pi)
:=
\sup_{\theta\in\Theta}
\frac{G-V(\pi\mid\theta)}
     {G-V^{\rm CL}(\theta)}.
\]
This criterion 
compares the policy's expected loss $G-V(\pi\mid\theta)$ with the
clairvoyant's minimum expected loss $G-V^{\rm CL}(\theta)$. The supremum over $\theta$
captures the worst-case state. In every state, the expected loss equals the expected calling cost plus $G$ times the probability that the window ends without an improvement. 

\vspace{1mm}
\noindent 
 \textbf{The construction of an Escalator policy.}
We next construct an Escalator policy that does not know the hidden  state  $\theta$. The policy requires four
pre-computed inputs: the configuration scores $(s_a^*)_{a\in\mathcal A}$, the
geometric budget scales $(B_j)_{j=1}^J$, the configuration-specific reach
functions $(R_a)_{a\in\mathcal A}$, and their maximum $R$ over all possible configurations. We first define
these quantities below.
 
 To compare calls across configurations, we first express their
effects on the probability of remaining unsuccessful on an
additive scale. Conditional on state $\theta \in \Theta$, the probability
that a prescribed sequence of calls $a_1,\ldots,a_n$ all fail is
$\prod_{i=1}^n(1-q_{a_i,\theta})$. Hence,
\[
-\log\left[\prod_{i=1}^n(1-q_{a_i,\theta})\right]
=
\sum_{i=1}^n-\log(1-q_{a_i,\theta}).
\]
Thus, $-\log(1-q_{a,\theta})$ is the contribution of one
$a$-call to the cumulative  failure probability
on a logarithmic scale. These contributions depend on the
unknown state. We therefore seek  state-independent
weights $(s_a)_{a\in \mathcal{A}}$ that represent the relative contributions of
different configurations, so that a sequence of calls can
be summarized by its total weight $\sum_{i=1}^n s_{a_i}$.
Given the success probabilities 
$q=(q_{a,\theta})_{a\in\mathcal A,\theta\in\mathcal D}$,
for any positive vector $\mathbf s=(s_a)_{a\in\mathcal A}$
normalized by $s_1=1$, define  
 \begin{align}
     \kappa_q(\mathbf s)
:= \max_{\theta \in \mathcal{D}} \max_{ a,b\in \mathcal{A}} \frac{-\log (1-q_{a,\theta})/ (-\log (1-q_{b,\theta}))}{s_a/s_b}  = 
\max_{\theta\in\mathcal D}
\frac{\max_{a\in\mathcal A}-\log(1-q_{a,\theta})/s_a}
     {\min_{a\in\mathcal A}-\log(1-q_{a,\theta})/s_a}.  \label{eq-definition-weighted score}
 \end{align}
 and $\mathbf s^* :=(s_a^*)_{a\in \mathcal{A}}
\in
\operatorname*{arg\,min}_{\mathbf s=(s_a)_{a\in \mathcal 
A}:\,s_1=1}\kappa_q(\mathbf s).$ 
If the minimizer is not unique, we select one according
to a fixed deterministic tie-breaking rule.  
The number of stages and the budget scale at each stage are defined by
\[
J
:=
\max\left\{
0,\left\lfloor\log_2(G/c_1)\right\rfloor
\right\},
\qquad
B_j:=2^jc_1,
\quad j=1,\ldots,J.
\] 
For each configuration $a\in\mathcal A$, define its reach function $R_a(B)$ and the upper-envelope reach function  $R(B)$ as follows: 
\[
R_a(B)
:=
s_a^*\left(\frac{B}{c_a}-1\right),
\quad R(B)
:=
\max_{a\in\mathcal A}R_a(B), \quad  \forall B\ge c_1.  
\]  
\textcolor{black}{Algorithm~\ref{alg:det-escalator} requires the scores $\mathbf s^*$, which depend on the success probabilities, as well as the call costs and reward $G$.
Its purpose here is to establish the existence of an Escalator policy satisfying the guarantee in Theorem~\ref{thm:escalator-guarantee}.
\textcolor{black}{Section~\ref{sec:numerics} evaluates OSCAR's accuracy and cost using the Escalator structure with directly tuned attempt counts.}}

  \textcolor{black}{Algorithm~\ref{alg:det-escalator} constructs such an Escalator policy.}  
This policy constructs its routing plan by gradually increasing how much it is prepared to spend. \textcolor{black}{It uses a positive weight \(s_a^* > 0\) for each configuration,} which helps compare the relative effectiveness of calls to different configurations across states. The variable \(u\) records the sum of these weights over calls already scheduled. At stage \(j\), \(B_j\) is set as the spending allowance for that stage, and the algorithm selects the configuration with
the largest target $R_a(B_j)$ and adds enough calls to bring
$u$ to at least this target.  
Doubling \(B_j\) allows the policy to consider progressively larger spending levels. We stop adding stages when the budget  $2^{J+1} c_1$  in the next stage would exceed the reward $G$, yielding the definition of $J$.

\begin{algorithm}[H]
 \caption{Escalator Policy for  General Sequential Improvement Search}
\label{alg:det-escalator}
\begin{algorithmic}
    \Require
    $(s_a^*)_{a \in \mathcal{A}}$, $(B_j)_{j=1}^{J}$,
    $(R_a(\cdot))_{a\in\mathcal A}$, and $R(\cdot)$. 

    \State $u \gets 0$

    \For{$j=1,\ldots,J$}
        \State
        $\displaystyle
        \alpha_j
        \gets
        \max\bigl\{
            a\in\mathcal A:
            R_a(B_j)=R(B_j)
        \bigr\}$

        \State
        $\displaystyle
        n_j
        \gets
        \left\lceil
            \frac{\bigl(R(B_j)-u\bigr)^+}
                 {s^*_{\alpha_j}}
        \right\rceil$

        \State
        $\displaystyle
        u \gets u+n_js^*_{\alpha_j}$
    \EndFor

    \State
    \Return
    $\displaystyle
    \pi^{\rm EP}
    :=
    \alpha_1^{n_1}
    \alpha_2^{n_2}
    \cdots
    \alpha_J^{n_J}$
\end{algorithmic}
\end{algorithm}

As established in the next proposition, this selection rule  ensures that every configuration switch increases the call cost, yielding  the Escalator structure. The proof is provided in Appendix~\ref{appendix-prop-structure}. 
\begin{proposition}[Escalator structure]
\label{prop:prior-free-escalator}   
 The configurations selected by Algorithm~\ref{alg:det-escalator} satisfy $\alpha_1\le\alpha_2\le\cdots\le\alpha_J$. 
\end{proposition}

We now state the main result of this subsection. 
\begin{theorem}[Competitive ratio guarantee]
\label{thm:escalator-guarantee}
The competitive ratio of the Escalator policy in Algorithm~\ref{alg:det-escalator}, denoted by $\pi^{\text{EP}}$, is bounded as follows:
\begin{align*}
    \operatorname{CR}(\pi^{\rm EP})
\le  4h(\kappa_q(\mathbf s^*)) < 5.2 \kappa_q(\mathbf s^*),
\end{align*}
where $h(x)
:=
\sup_{z>0}(1+z^{-1})(1-e^{-x z})$ for $x \ge 1$. 
\end{theorem}

The competitive guarantee in Theorem~\ref{thm:escalator-guarantee} compares the policy's expected loss with
the clairvoyant's expected loss and holds uniformly over the hidden state $\theta$.  
To see how the more explicit  upper bound $5.2\kappa_q(\mathbf s^*)$ is derived, for any $x\ge1$, taking $z\downarrow0$ in the
definition of $h(x)$ gives $h(x)\ge x$. On the other hand,
by changing the variable in the \textcolor{black}{supremum}, we get 
\begin{align*}
h(x)
 =
\sup_{z>0}
\left(1+{x}/{z}\right)(1-e^{-z}) \le
x\sup_{z>0}
\left(1+ 1/z\right)(1-e^{-z})
=
xh(1), 
\end{align*}
where the inequality follows from $x\ge1$.
Let $g_\star:=h(1)$. Since the derivative of $(1+z^{-1})(1-e^{-z})$ with respect to $z$ is $(e^{-z}(z^2+z+1)-1)z^{-2}$, the \textcolor{black}{supremum} is achieved at $z_0$ satisfying $e^{z_0}=z_0^2+z_0+1$. Numerical evaluation gives $z_0\approx1.793$   and $g_\star\approx1.298$. Consequently,
\begin{align}
\operatorname{CR}(\pi^{\rm EP})
 \le
4h\!\left(\kappa_q(\mathbf s^*)\right)  \le
4g_\star\kappa_q(\mathbf s^*)<5.2\,\kappa_q(\mathbf s^*).
\label{eq:escalator-explicit-bound}
\end{align}

Note that the upper bound
$4h(\kappa_q(\mathbf s^*))$ is non-decreasing in
$\kappa_q(\mathbf s^*)$. Its smallest possible value,
$4g_\star\approx5.194$, is attained when
$\kappa_q(\mathbf s^*)=1$, meaning that
$-\log(1-q_{a,\theta})/s_a^*$ is identical across configurations
for each fixed state $\theta$. For a general sequence of success probabilities  $(q_{a,\theta})_{a\in \mathcal{A},\theta \in \Theta}$, however, such equality need not be attainable
simultaneously for all states. \textcolor{black}{We choose $\mathbf s^*$ to minimize $\kappa_q(\mathbf s)$, thereby tightening the theoretical bound.} The proof of Theorem~\ref{thm:escalator-guarantee} is in Appendix~\ref{appendix-thm-guarantee}.

For Type~1 windows, Section~\ref{sec:prior} provides the intuition for continuing the search until improvement.
Suppose $G$ is large enough that repeating the final configuration until improvement is worthwhile in every improvable state.
It is easy to show that extending the last stair weakly reduces expected loss under this condition, so we omit the details.
This observation is consistent with the analysis in Section~\ref{sec:prior} and supports the Type~1 continuation rule in Section~\ref{sec:policy}.

\subsection{The Policy OSCAR Runs}\label{sec:policy}

\textcolor{black}{In this subsection, we use the escalation and stopping structure from Sections~\ref{sec:prior} and~\ref{sec:competitive} to specify OSCAR's routing policy. \textcolor{black}{We retain the configurations' cost ordering and tune the attempt counts directly in Section~\ref{sec:numerics}, without deriving them from estimated per-call costs, success probabilities or a prior over difficulty.} We also incorporate the revision strategy motivated by Section~\ref{sec:observations}.}

\emph{Which configuration.} \textcolor{black}{\textcolor{black}{Each window progresses from cheaper to more expensive configurations, escalating when the current configuration exhausts its allotted attempts.} As a rule of thumb, harder problems relative to the models' capabilities call for more attempts per configuration.}

\emph{When to stop.} \textcolor{black}{After the menu is exhausted without improvement, Type~2 ends the session and Type~1 repeats the final configuration until an improvement is certified. Theorem~\ref{thm:escalator} and the extension in Section~\ref{sec:competitive} support this distinction under their stated conditions. A large external stopping limit can cap the Type~1 extension in practice. No run in Section~\ref{sec:numerics} reached this limit.}

\emph{What to start from.} \textcolor{black}{The initial formulation uses a strong Coder (Section~\ref{sec:framework}). Each window starts from the incumbent, with later attempts revising the latest draft. Following Section~\ref{sec:observations}, the loop reverts to the incumbent when the Coder's tier increases. During the Type~1 extension, it also returns to the incumbent after every $r$ uncertified attempts, where $r$ is the restart interval. We set $r=3$ for the experiments in Section~\ref{sec:numerics}.}

\textcolor{black}{See Appendix~\ref{ec:num-models} for details of the experimental settings.}

\section{Numerical Study}\label{sec:numerics}

We examine whether OSCAR improves the accuracy of \textcolor{black}{small LLMs} and whether its improvement loop, comprising the Reviewer and the Escalator, reduces the cost of obtaining accurate solutions. \textcolor{black}{We use the Qwen3.5- and Qwen3.6-Flash models (both 35B-A3B, hereafter Qwen3.5/3.6)} \citep{qwen2026qwen35, qwen2026qwen36}. \textcolor{black}{Each model has 35 billion parameters, of which 3 billion are active per token.} The experiments use cloud APIs, although both models support local deployment on a single GPU. These inexpensive \textcolor{black}{LLMs} are not tuned for optimization modeling. Their limited one-shot performance allows us to assess the framework's contribution. Appendix~\ref{ec:numerics} provides the experimental details.

\subsection{\textcolor{black}{Comparison with One-Shot and Frontier Coding Agents}}\label{sec:num-setup}\label{sec:num-main}

\textcolor{black}{\emph{Problems and data.}} FrontierOR \citep{frontieror2026} contains 180 optimization problems drawn from papers in leading OR journals, each with a description, instances, a reference solution and a feasibility checker. We select one problem from each of the four classes (vehicle routing, graph optimization, packing and cutting stock, and scheduling) that contain more than fifteen problems, together with one quadratic optimization problem. We refer to them as Routing, Graph, Packing, Scheduling and Quadratic. \textcolor{black}{Table~\ref{tab:num-problems} gives their benchmark IDs and formulation types.} For each problem, we assess the generated formulations by checking their solutions on the smallest available instance.

For Simulator construction, we supply roughly ten complete labeled decision examples per problem. These include benchmark reference solutions from other instances, feasible but suboptimal plans, and perturbed decisions that violate specific operating rules. The optimal reference solution for the target instance is excluded from these examples.

We limit the study to five problems for two reasons. OSCAR requires additional inputs, including decision syntax and labeled decision examples, which need problem-specific preparation. \textcolor{black}{All five problems required data processing before Simulator construction, mainly to map reported decisions to instance elements. Screening also revealed ambiguous descriptions, unstated checker input formats, and inconsistencies among problem descriptions, reference formulations and solutions, and feasibility checkers. We retained problems that could be evaluated without changing these benchmark components, and the five selected descriptions remain as released. The problem data, certified Simulators, labeled decision examples and code are available at \url{https://github.com/Sandrone-lab/OSCAR}.}

\begin{table}[htbp]
\centering
\color{black}
\small
\caption{The five problems.}\label{tab:num-problems}
\begin{tabular}{llll}
\toprule
Problem class & FrontierOR class (problems) & Problem ID & Formulation \\
\midrule
Routing    & Vehicle routing \& TSP (30)  & colombi2017  & MIP        \\
Graph      & Graph optimization (19)      & mehrotra1996 & IP         \\
Packing    & Packing / cutting stock (17) & letelier2022 & Binary IP  \\
Scheduling & Scheduling (17)              & elci2022     & Stochastic \\
Quadratic  & Quadratic optimization (3)   & buchheim2018 & QP / MIQP  \\
\bottomrule
\end{tabular}
\end{table}

\emph{Models and policy.} \textcolor{black}{The four configurations, written as (Reviewer, Coder) pairs in cost order, are} (Qwen3.5, Qwen3.5), (Qwen3.6, Qwen3.5), (Qwen3.5, Qwen3.6) and (Qwen3.6, Qwen3.6). We use the policy of Section~\ref{sec:policy}, with restart interval $r=3$. \textcolor{black}{We first test one attempt per configuration on all five problems, $n=(1,1,1,1)$, abbreviated as $n=1$. \textcolor{black}{For Routing, accuracy is 72\% at $n=1$, so we also test a larger allocation, $n=(1,1,3,3)$.} We do not test this larger allocation on the other four problems, which already reach 95\% to 100\% accuracy at $n=1$ (Table~\ref{tab:num-main}).}

\emph{Baselines and measures.} A one-shot run consists of the first Coder call of an OSCAR session followed by syntax repair, with either model and without a Simulator or Reviewer. We also run two frontier coding agents, Codex with GPT-6 Astra and Claude Code with Fable~5.1, five times per problem. Each receives only the description and instance. This compares complete workflows, with OSCAR additionally using labeled decision examples and Simulator feedback. \textcolor{black}{We limit the comparison to five runs per problem for two reasons. These agents are costly, and \textcolor{black}{updates by the providers to their closed models} may change their performance, making results from more extensive testing difficult to reproduce later. We use these runs primarily to compare the agents' costs with OSCAR's.}

We refer to the benchmark's feasibility checker as the trusted checker to distinguish it from the Simulator. A run is correct if its final solution is feasible according to the trusted checker and reaches the reference optimum. We price model calls at list rates and include Simulator construction in OSCAR's cost (Appendix~\ref{ec:num-cost}).

\begin{table}[htbp]
\centering

\caption{Runs reaching the reference optimum (\%), with \textcolor{black}{cost per run in US dollars in parentheses.} \textcolor{black}{For each problem, we run each one-shot model and tested OSCAR setting 100 times and each frontier agent five times.}}\label{tab:num-main}
\setlength{\tabcolsep}{2.5pt}
\begin{tabular}{lcccccc}
\toprule
 & \multicolumn{2}{c}{One-shot} & \multicolumn{2}{c}{OSCAR} & \multicolumn{2}{c}{\textcolor{black}{Frontier coding agents}} \\
\cmidrule(lr){2-3}\cmidrule(lr){4-5}\cmidrule(lr){6-7}
Problem class & Qwen3.5 & Qwen3.6 & $n=1$ & $n=(1,1,3,3)$ & \textcolor{black}{\shortstack{Codex with\\GPT-6 Astra}} & \shortstack{Claude Code\\with Fable~5.1} \\
\midrule
Routing    & 7 (\$0.02)  & 18 (\$0.04) & 72 (\$0.17)  & 95 (\$0.28) & \textcolor{black}{5/5 (\$0.65)} & 5/5 (\$1.23) \\
Graph      & 77 (\$0.01) & 89 (\$0.01) & 100 (\$0.05) & --          & \textcolor{black}{5/5 (\$0.50)} & 5/5 (\$0.65) \\
Packing    & 7 (\$0.01)  & 12 (\$0.02) & 97 (\textcolor{black}{\$0.27})  & --          & \textcolor{black}{5/5 (\$0.64)} & 5/5 (\$1.43) \\
Scheduling & 4 (\$0.01) & 35 (\$0.04) & 95 (\$0.24)  & --          & \textcolor{black}{5/5 (\$0.46)} & 5/5 (\$1.13) \\
Quadratic  & 49 (\$0.01) & 84 (\$0.01) & 100 (\$0.13) & -- & \textcolor{black}{5/5 (\$0.74)} & 5/5 (\$1.20) \\
\bottomrule
\end{tabular}
\end{table}

\emph{Results.} Averaged over the five problems, a single attempt reaches the reference optimum in 48\% of runs with Qwen3.6 and 29\% with Qwen3.5. With one attempt per configuration, OSCAR reaches it in 72 to 100\% of runs. Routing is the most difficult case: at $n=1$, the models often remain at a feasible but suboptimal formulation. Allocating more attempts to the two more expensive configurations raises accuracy to 95\% for an additional \$0.11 per run, with each improvement certified by the Simulator.

\textcolor{black}{Averaging costs across the five problems, with Routing at $n=(1,1,3,3)$, gives \$0.193 per run for OSCAR, including Simulator construction, \$0.600 for Codex and \$1.127 for Claude Code. The agent costs are 3.1 and 5.8 times OSCAR's cost, respectively. Appendix~\ref{ec:num-cost} reports details of the cost accounting.}

\subsection{Contribution of the Improvement Loop}\label{sec:num-escalator}

To assess the improvement loop's contribution to the results in Table~\ref{tab:num-main}, we retain the Simulator but remove the Reviewer, incumbent revision and routing. The resulting resampler draws one-shot formulations independently, retains the best solution accepted by the Simulator, and stops after $k$ consecutive draws without improvement. Each run is charged once for Simulator construction and the structure constructor, plus the measured Coder and syntax-repair costs of each draw. \textcolor{black}{We estimate the resampler's accuracy and expected cost from the one-shot experiments in Section~\ref{sec:num-main}, which comprise 100 runs per model and problem. No additional LLM calls are made.} \textcolor{black}{Table~\ref{tab:num-resample} reports the smallest $k$ whose estimated accuracy is within $1\%$ of OSCAR's most accurate tested setting, $n=(1,1,3,3)$ on Routing and $n=1$ elsewhere.}

\begin{table}[htbp]
\centering

\caption{\textcolor{black}{Estimated cost of resampling to within $1\%$ of OSCAR's accuracy. Parentheses show cost ratios relative to OSCAR.}}\label{tab:num-resample}
\begin{tabular}{lcccccc}
\toprule
 & \multicolumn{2}{c}{OSCAR} & \multicolumn{2}{c}{Resampling, Qwen3.5} & \multicolumn{2}{c}{Resampling, Qwen3.6} \\
\cmidrule(lr){2-3}\cmidrule(lr){4-5}\cmidrule(lr){6-7}
Problem class & Correct (\%) & Cost per run & $k$ & Cost per run & $k$ & Cost per run \\
\midrule
Routing    & 95  & \$0.28 & 27 & \$0.53 (1.9$\times$) & 10 & \$0.58 (2.1$\times$) \\
Graph      & 100 & \$0.05 & 3  & \$0.03 (0.7$\times$) & 3  & \$0.04 (0.8$\times$) \\
Packing    & 97  & \textcolor{black}{\$0.27} & 39 & \textcolor{black}{\$0.55 (2.0$\times$)} & 24 & \textcolor{black}{\$0.66 (2.5$\times$)} \\
Scheduling & 95  & \$0.24 & 50 & \$0.80 (3.3$\times$) & 6  & \$0.39 (1.6$\times$) \\
Quadratic  & 100 & \$0.13 & 7 & \$0.14 (1.0$\times$) & 3 & \$0.12 (0.9$\times$) \\
\bottomrule
\end{tabular}
\end{table}

On Routing, Packing and Scheduling, where one-shot success is uncommon, resampling to match OSCAR's accuracy costs 1.6 to 3.3 times as much. On Graph and Quadratic, where one-shot success is more frequent, resampling costs 0.7 to 1.0 times as much. These results suggest that verified resampling can be economical when one-shot success is common, while diagnosis and revision become more valuable when it is rare.

\FloatBarrier
\section{Conclusion and Future Research Directions}\label{sec:conclusion}

\textcolor{black}{The OSCAR framework uses a certified Simulator to compare candidates and continues searching beyond feasibility through iterative improvement and cost-ordered escalation.} \textcolor{black}{Under the assumptions of Section~\ref{sec:prior}, an optimal policy has an Escalator structure in the setting with two configurations, two difficulty levels and a known prior. For general menus, Algorithm~\ref{alg:det-escalator} constructs a schedule that follows the same Escalator structure and achieves a prior-free competitive guarantee.}

\textcolor{black}{On five benchmark instances, OSCAR achieves 95\% to 100\% accuracy at the reported settings with two small LLMs. Its token costs are lower than those of Codex and Claude Code in our limited frontier-agent comparison. These findings support assessing available models with verification and iterative improvement before upgrading to frontier models for optimization modeling. \textcolor{black}{Each of these LLMs can be deployed locally on a single GPU. Firms can choose local or cloud deployment to suit their cost and confidentiality requirements.} Maintaining \textcolor{black}{labeled decision examples} alongside plain-language descriptions helps clarify business rules and supports Simulator construction. OSCAR's simple operating rules and adjustable settings also help firms benefit as LLM capabilities and prices change.}

Three directions remain for future research. \textcolor{black}{The first is to extend our focus from solution correctness on a single target instance to formulation correctness across all instances consistent with the problem description. This calls for a separate layer of the framework that develops and validates comprehensive verification procedures independently of formulation generation.}
A second direction is to improve computational performance. Formulation choices can substantially affect solver performance \citep{simchilevi2026democratizing}. Future work could combine solve time with Simulator evaluations to guide the search toward formulations that are correct and efficient to solve.
A third direction is to reduce manual data preparation. Future work could adapt the Simulator construction procedure to build and check data-reading routines, using tests that verify whether input data are interpreted correctly.

\bibliographystyle{informs2014}
\bibliography{reference}

\ECSwitch
\makeatletter
\let\ECOriginalNormalSize\AAnormalsizeXI
\renewcommand\AAnormalsizeXI{%
  \ECOriginalNormalSize
  \fontsize{11}{18}\selectfont
  \baselineskip 18pt plus 1pt minus 1pt
  \@bls=18pt\relax}
\makeatother
\OneAndAHalfSpacedXI

\ECHead{E-Companion to ``OR for AI That Does OR: Routing LLMs up the Escalator inside the OSCAR Framework''}

\section{Pseudocode for Simulator Construction and the OSCAR Framework}\label{ec:empirics}

\textcolor{black}{Algorithms~\ref{alg:simulator} and~\ref{alg:oscar} describe Simulator construction and an OSCAR session. We first explain the optional \textcolor{black}{partial decision examples} and the checks on violation messages.}

\subsection*{\textcolor{black}{Supplying partial decision examples}}

Accepting complete feasible examples does not establish that every constraint is checked correctly, especially when no example approaches violating it. A \textcolor{black}{partial decision example} isolates a constraint in a fragment, such as a route or period, labeled as respecting or violating it. Pairs that differ in one respect are especially useful. A planner can supply fragments, or an LLM can draft them. Labels must be confirmed by the firm or derived from a perturbation of \textcolor{black}{an existing labeled decision example}, rather than accepted solely from the formulation models. During Simulator construction, partial examples carry a flag so that intentionally omitted decisions are not treated as errors. This flag is not used in OSCAR runs.

\subsection*{Certifying violation messages}

Correct rejection labels alone do not provide enough feedback for revision. The Reviewer also needs an explanation that the Coder can act on.

\textcolor{black}{A separate LLM scores whether each message identifies the rule, location and magnitude without access to the code. The mean score must reach $b$. The scorer supplies replacement wording, which is returned to the writer instead of scores. These repairs may change wording only. Any repair that changes a label is discarded.}

\textcolor{black}{Construction starts with the cheapest writer and advances after its round budget is exhausted. Some problems outside the benchmark study required human help during construction, mainly with formatting, suggesting an application of OSCAR to checking data-processing routines.}

\begin{breakablealgorithm}
\caption{Building the Simulator}
\label{alg:simulator}
{\color{black}
\begin{algorithmic}[1]
\Require description, target data, decision syntax, \textcolor{black}{labeled decision examples and optional labeled partial decision examples}, objective tolerance $\varepsilon$, writer models in cost order with round budgets, and message-quality threshold $b$
\For{each writer model in cost order}
  \State Ask it to write a Simulator from the description and decision syntax.
  \Repeat
    \State Repair errors that prevent execution.
    \State Check all examples for the correct \textsc{feasible} or \textsc{infeasible} label. \textcolor{black}{Check a feasible example's objective within $\varepsilon$ only if a value is reported.}
    \State Empty each field in turn and randomly replace fields with values of the wrong type. Require \textsc{malformed}, while accepting harmless edits such as unused extra fields.
    \State Check only each marked fragment. Match \textcolor{black}{its} \textsc{partial-feasible} or \textsc{partial-infeasible} label, name the violated constraint when rejecting, and do not flag omitted parts.
    \State \textcolor{black}{If a feasible example reports an objective, change that value by more than $2\varepsilon$, leaving its decisions unchanged. Require \textsc{inconsistent-objective}. Otherwise skip this check.}
    \If{every check passes}
      \State Have a separate LLM score rejection messages for rule, location and magnitude, and supply replacement wording.
      \If{the mean score is at least $b$}
        \State \Return the Simulator, fixed for subsequent OSCAR runs
      \EndIf
    \EndIf
    \State \textcolor{black}{Ask the writer to revise the code using the failing cases and explanations.} For message-only failures, supply replacement wording and preserve all labels.
  \Until{this model's round budget is reached}
\EndFor
\State \Return failure
\end{algorithmic}
}
\end{breakablealgorithm}

\begin{breakablealgorithm}
\caption{One OSCAR session}
\label{alg:oscar}
{\color{black}
\begin{algorithmic}[1]
\Require description, decision syntax, \textcolor{black}{labeled decision examples}, target instance, cost-ordered (Reviewer, Coder) pairs, attempts per pair, and restart interval $r$
\State Preprocess the description into \textcolor{black}{a parameter schema} and build the Simulator by Algorithm~\ref{alg:simulator}.
\State \textcolor{black}{Ask a high-tier Coder for an initial formulation. Run, repair and evaluate it as below to initialize the incumbent and its feedback.}
\Repeat
  \State Begin a window with the incumbent as the starting formulation.
  \For{each pair in cost order and each attempt allowed at it}
    \State When the Coder's tier increases, return to the incumbent.
    \State Give the Reviewer the starting formulation and its \textcolor{black}{feedback}, and obtain revision advice.
    \State Generate and run the revision. Repair execution errors with a low-tier LLM, escalating one tier after two failed rounds.
    \State \textcolor{black}{A formulation returning no solution has the lowest rank and uses solver status as feedback, without a Simulator call.} Otherwise evaluate the solution. Send \textsc{malformed} outputs back to repair without ranking them.
    \If{the candidate improves on the incumbent: \textcolor{black}{a readable solution when the incumbent returns none,} a feasible solution with a better objective, feasibility from an infeasible incumbent, or fewer violated rules while both are infeasible}
      \State Update the incumbent and \textcolor{black}{immediately start a new window}
    \Else
      \State Use this draft and its \textcolor{black}{feedback} for the next attempt
    \EndIf
  \EndFor
  \State \textcolor{black}{If the menu ends without improvement in a Type~1 window (opened without a feasible solution),} repeat the final pair until improvement, returning to the incumbent every $r$ uncertified attempts.
\Until{a Type~2 window (opened feasible) exhausts its menu without improvement}
\State \Return the incumbent
\end{algorithmic}
}
\end{breakablealgorithm}

\section{Proofs of the Known-Prior Optimality Result (Section \ref{sec:prior})}\label{ec:proofs-prior}

\subsection{Proof of Lemma~\ref{lemma-sec 4.2-Bellman infinite horizon}} \label{appendix-proof-lemma-Bellman}
Recall that $J_k(\beta)$  is the 
maximum expected net reward from belief $\beta$
with at most $k$ additional calls permitted.
We  establish the following inequalities: 
\begin{equation}
0\leq J_\infty(\beta)-J_k(\beta)
\leq G(1-\beta_o)\Big(1-\min_{a \in \mathcal{A}, \theta \in \mathcal{D}}q_{a,\theta}\Big)^k.
\label{eq:infinite-horizon-approximation}
\end{equation}
Since $0< q_{a,\theta} < 1$ for any $a\in \mathcal{A}$ and $\theta \in \mathcal{D}$, the above inequality   implies $\lim_{k \to \infty} J_k(\beta) =J_{\infty}(\beta)$. By letting $k\to \infty$ in each side of the finite-horizon Bellman equation in  \eqref{eq:belief-bellman}, we have  
\begin{align*}
     J_\infty(\beta) 
    = \lim_{k \to \infty} J_k(\beta) 
     &= \lim_{k \to \infty}  \max\bigl\{0, \ (\mathcal{L}_A J_{k-1})(\beta),\ (\mathcal{L}_B J_{k-1})(\beta) \bigr\}  \nonumber \\
     & = \max\{ 0, \ (\mathcal{L}_A J_{\infty})(\beta),\ (\mathcal{L}_B J_{\infty})(\beta) \},
\end{align*}
where the last identity follows because
$\lim_{k \to \infty} J_{k-1}(T_a(\beta)) = J_\infty(T_a(\beta))$ implies
$\lim_{k \to \infty}  (\mathcal L_aJ_{k-1})(\beta) = (\mathcal L_aJ_\infty)(\beta)$
for each $a\in\{A,B\}$, and taking the maximum over finitely
many terms preserves the convergence.

To see the lower bound in \eqref{eq:infinite-horizon-approximation}, note that every plan permitting at most $\textcolor{black}{k}$
calls is feasible in the infinite-horizon problem,
which gives  $J_{\infty}(\beta)   \ge J_k(\beta)$ by the definition of $J_{\infty}(\beta) $.
For the upper bound, truncate any admissible infinite-horizon  plan $\pi$
after $ k $ calls and denote this truncated policy by $\pi^{(k)}$. Then we have the following inequality for any    belief $\beta$:
\begin{align}
    V_{\beta}(\pi ) -  J_k(\beta) \le   V_{\beta}(\pi )  - V_{\beta}(\pi^{(k)}  ) \le G(1-\beta_o)\Big(1- \min_{a \in \mathcal{A}, \theta \in \mathcal{D}} q_{a,\theta}\Big)^k. \label{ineq-proof-lemma-Bellman}
\end{align}
Here, the first inequality holds because from its definition, $J_k(\beta)$ is the maximum expected net reward with $k$ available calls, which is no less than the expected net reward of $\textcolor{black}{\pi^{(k)}}$.  
The second inequality holds because any additional reward under
$\pi$ requires an improvable state $\theta\in \mathcal D$ and failure of the first $k$
calls, an event with  the following probability: 
\begin{align*}
\sum_{\theta \in \mathcal{D}} \beta_{\theta} \prod_{i=1}^k (1-q_{a_i,\theta})  \mathrel{\textcolor{black}{\le}}  \Big(1-\min_{a \in \mathcal{A},\theta \in \mathcal{D}}q_{a,\theta}\Big)^k \cdot \sum_{\theta \in \mathcal{D}} \beta_{\theta}  = 
(1-\beta_o)
\Big(1-\min_{a\in\mathcal A,\theta\in\mathcal D}q_{a,\theta}\Big)^k. 
\end{align*}
On this event, continuing can yield at most one additional
reward $G$ and incurs nonnegative calling costs.  
Taking the supremum over all
routing plans in \eqref{ineq-proof-lemma-Bellman}
gives the upper bound in~\eqref{eq:infinite-horizon-approximation}. \hfill \qed

\subsection{Proof of Lemma \ref{lem:bellman-interchange}} \label{appendix-lemma-adjacent exchange}
For $i,j\in\{A,B\}$ with $i\neq j$, expanding the two Bellman
operators gives 
\begin{align}
(\mathcal L_i\mathcal L_j f)(\beta)
={}&-c_i+Gp_i(\beta) +(1-p_i(\beta))
\Big[
-c_j+Gp_j(T_i(\beta)) 
+(1-p_j(T_i(\beta)))
f(T_j(T_i(\beta)))
\Big] \nonumber \\
={}&
\underbrace{-c_i-(1-p_i(\beta))c_j}_{\text{expected calling cost}} +
\underbrace{G\Big[p_i(\beta)
+(1-p_i(\beta))p_j(T_i(\beta))\Big]}_
{\text{expected success reward}} \nonumber \\
&+
\underbrace{(1-p_i(\beta))(1-p_j(T_i(\beta)))
f(T_j(T_i(\beta)))}_
{\text{continuation value}}. \label{eq-expansion value difference}
\end{align} 

We first show that the expected success reward and continuation-value
components in the second and third terms of \eqref{eq-expansion value difference} 
are invariant to the order of the two calls. By the
posterior update rule, 
\begin{align}
 p_i(\beta)
+(1-p_i(\beta))p_j(T_i(\beta)) 
=
\sum_{\theta\in\Theta}
\beta_\theta
\left[
q_{i,\theta}+(1-q_{i,\theta})q_{j,\theta}
\right] 
=
\sum_{\theta\in\Theta}
\beta_\theta
\left[
1-(1-q_{i,\theta})(1-q_{j,\theta})
\right]. \nonumber 
\end{align} 
Since the final expression is symmetric in $i$ and $j$, the two orders
generate the same expected success reward. 
Similarly, the probability that both calls fail is 
\begin{align}
 (1-p_i(\beta))(1-p_j(T_i(\beta))) 
=
\sum_{\theta\in\Theta}
\beta_\theta(1-q_{i,\theta})(1-q_{j,\theta}), \nonumber 
\end{align} 
which is also symmetric in $i$ and $j$. Conditional on both failures,
the posterior belief is
\[
T_j(T_i(\beta))_\theta
=
\frac{
\beta_\theta(1-q_{i,\theta})(1-q_{j,\theta})
}{
\sum_{\theta'\in\Theta}
\beta_{\theta'}
(1-q_{i,\theta'})(1-q_{j,\theta'})
},
\qquad \theta\in\Theta.
\]
Hence, $T_j(T_i(\beta))=T_i(T_j(\beta))$, 
and the continuation value in \eqref{eq-expansion value difference} 
is also invariant to the
order.

It then only suffices to bound 
the expected calling-cost difference. In this case, 
\begin{align}
 (\mathcal L_A\mathcal L_B f)(\beta)
-(\mathcal L_B\mathcal L_Af)(\beta) 
&=
\left[-c_A-(1-p_A(\beta))c_B\right]
-\left[-c_B-(1-p_B(\beta))c_A\right] \nonumber\\
& =c_Bp_A(\beta)-c_Ap_B(\beta). \nonumber 
\end{align} 
Using \(p_a(\beta)=\sum_{\theta\in\Theta}\beta_\theta q_{a,\theta}\),
\(q_{A,o}=q_{B,o}=0\), and
\(c_a=q_{a,\theta}\gamma_{a,\theta}\), we obtain 
\begin{align}
c_Bp_A(\beta)-c_Ap_B(\beta)
&=
\sum_{\theta\in\{e,h\}}
\beta_\theta
\bigl(c_Bq_{A,\theta}-c_Aq_{B,\theta}\bigr) =
\sum_{\theta\in\{e,h\}}
\beta_\theta q_{A,\theta}q_{B,\theta}
\bigl(\gamma_{B,\theta}-\gamma_{A,\theta}\bigr), \nonumber 
\end{align} 
which proves \eqref{eq-AB-BA}. 

Since the prior has full support, every reachable posterior satisfies
$\beta_e>0$. Dividing the last expression in \eqref{eq-AB-BA} by $\beta_e$ yields
$D(\lambda)$, with $\lambda=\beta_h/\beta_e$. After a failed
call to model $a$,
 \begin{align}
   \lambda(T_a(\beta))
=
\lambda(\beta)
\frac{1-q_{a,h}}{1-q_{a,e}}
>
\lambda(\beta), \label{ineq-monotone lambda}
 \end{align}
where the inequality follows from Assumption~\ref{ass:two-model-structure}(i). Finally,
Assumption~\ref{ass:two-model-structure}(iii) excludes $\gamma_{B,e}<\gamma_{A,e}$ and $\gamma_{B,h}>\gamma_{A,h}$ 
from holding simultaneously. Thus, the affine function
$D(\lambda)$ cannot cross zero from negative to positive as
$\lambda$ increases. Along an all-failure path, it can therefore
cross zero at most once, and only from nonnegative to nonpositive. \hfill \qed 
 
\subsection{Proof of Lemma \ref{lem:terminal-replacement}}\label{appendix-lemma-terminal replacement}
We note that Assumption~\ref{ass:two-model-structure}(ii) implies $p_B(\beta)\ge p_A(\beta)$. Moreover, the two inequalities in the
statement imply $G\ge c_A/p_A(\beta)\ge c_B/p_B(\beta)$. Therefore,
\[
  Gp_B(\beta)-c_B
  =p_B(\beta)\left(G-\frac{c_B}{p_B(\beta)}\right)
  \ge p_A(\beta)\left(G-\frac{c_B}{p_B(\beta)}\right)
  \ge Gp_A(\beta)-c_A.
\]
Since the replaced call is terminal, no further cost is incurred. This completes the proof of Lemma~\ref{lem:terminal-replacement}. 
\hfill \qed 

\subsection{Proof of the Escalator Structure in Equation~\eqref{eq:escalator-form} of Theorem~\ref{thm:escalator}}\label{appendix-thm-escalator}  
To prove equation~\eqref{eq:escalator-form}, we first establish a counterpart result in  a finite-horizon problem with at most $ k <\infty$ calls permitted in  the remaining improvement window. Then we extend the result to the infinite-horizon problem with no restriction on the number of available calls. 

\vspace{1mm}
\noindent
\underline{Step  1: analysis for the finite-horizon problem.}  
For the fixed posterior belief $\beta$, suppose at most $k$ calls are allowed to make, and denote the
optimal deterministic routing policy  by the following 
action sequence along its all-failure path: 
\begin{equation}
(a_1,\ldots,a_n;\mathrm{stop}),
\qquad
0 \le n \le k. 
\label{eq:optimal-failure-sequence}
\end{equation} 
We will prove that the optimal policy can be chosen in the following form: 
\begin{align} 
\bigl(
\underbrace{A,\ldots,A}_{n_A\text{ calls}},
\underbrace{B,\ldots,B}_{n_B\text{ calls}};
\mathrm{stop}
\bigr), \quad 0 \le n_A+n_B \le k.\label{eq:escalator-form-finite proof}
\end{align}

We first show that an optimal plan can be chosen to have the following form: 
\begin{align}
\bigl(
\underbrace{A,\ldots,A}_{r \text{ calls}},
\underbrace{B,\ldots,B}_{s \text{ calls}}, \underbrace{A,\ldots,A}_{u \text{ calls}};
\mathrm{stop}
\bigr), \quad 0 \le r+s+u \le k. 
\label{eq:three-block-form}
\end{align}
For convenience, we use $A^r B^s A^u$ to denote the form in~\eqref{eq:three-block-form}.  For each call \(t\) on the all-failure path, let \(\beta_t\) denote the 
posterior belief immediately before call \(t\), with $\beta_1=\beta$, and define
\[
\lambda_t:=\frac{\beta_{t,h}}{\beta_{t,e}}.
\] 
By Lemma~\ref{lem:bellman-interchange} and the optimality of the continuation, any adjacent transition 
must satisfy
\[
a_ta_{t+1}=AB \ \Longrightarrow\ D(\lambda_t)\geq 0,
\qquad
a_ta_{t+1}=BA \ \Longrightarrow\ D(\lambda_t)\leq 0.
\] 
Along the all-failure path, by inequality \eqref{ineq-monotone lambda}, \(\lambda_t\) is strictly increasing. Moreover, 
by Assumption~\ref{ass:two-model-structure}(iii), \(D(\lambda)\) cannot change sign from negative to 
positive as \(\lambda\) increases. Consequently, a \(BA\) transition cannot 
precede a later \(AB\) transition. To see this, suppose that such transitions 
occurred at positions \(i<j\). 
If \(D(\cdot)\not\equiv0\), optimality would require $D(\lambda_i)\leq 0$ and $D(\lambda_j)\geq 0,$ 
which contradicts the single-crossing property because
\(\lambda_j>\lambda_i\). If
\(D(\cdot)\equiv0\), Lemma~\ref{lem:bellman-interchange} implies that every adjacent interchange is
value-neutral, so the calls can be reordered without changing the value.  Because transitions in a binary action sequence alternate, the absence of a
\(BA\) transition followed by a later \(AB\) transition implies that there
exists a value-equivalent optimal continuation with at most three blocks,
ordered \(A\), \(B\), and \(A\). Hence,
the optimal routing plan has the form in \eqref{eq:three-block-form}.

If $s=0$, the sequence in \eqref{eq:three-block-form} reduces to
$A^{r+u}$, so the result follows with $n_A=r+u$ and $n_B=0$. If
$u=0$, it already has the Escalator form $A^rB^s$.  It remains to eliminate the terminal $A$-block when
$s> 0$ and $u>0$.
Consider the $BA$ transition between the $B^s$-block and the terminal
$A^u$-block. The calls at positions $r+s$ and $r+s+1$ form this
$BA$ pair, and $\beta_{r+s}$ is the posterior belief immediately
before the pair. Optimality and
Lemma~\ref{lem:bellman-interchange} imply
$D(\lambda_{r+s})\leq0.$ By inequality~\eqref{ineq-monotone lambda}, we further have 
\begin{equation}
D\bigl(\lambda_t\bigr)\leq 0, \quad \forall t = r+s, \ldots, r+s+u. 
\label{eq:D-boundary}
\end{equation} 
Since $\beta_{r+s+u}$ is the posterior belief immediately before the last
call to $A$, and this call is followed by stopping if it fails, we then have 
\begin{equation}
Gp_A(\beta_{r+s+u})-c_A\geq 0.
\label{eq:terminal-A-profitable}
\end{equation}
Moreover, $D(\lambda_{r+s+u})\leq 0$ according to \eqref{eq:D-boundary}. Since the difference in
Lemma~\ref{lem:bellman-interchange} satisfies $c_Bp_A(\beta_{r+s+u})-c_Ap_B(\beta_{r+s+u})
=
\beta_{r+s+u,e}D(\lambda_{r+s+u}),$
we obtain
\begin{equation}
\frac{c_B}{p_B(\beta_{r+s+u})}
\leq
\frac{c_A}{p_A(\beta_{r+s+u})}.
\label{eq:terminal-cost-effectiveness}
\end{equation}
Therefore, the two conditions in Lemma~\ref{lem:terminal-replacement} are satisfied. Hence, replacing the terminal call to $A$ by a call to
$B$ does not decrease the expected value. Because the original policy
is optimal, the modified policy 
$A^{r}B^{s}A^{u-1}B$ is also optimal.

We next move the newly introduced  $B$-call in the last position leftward through the remaining
$u-1$ terminal $A$-calls. When this $B$-call is adjacent to the
$j$th $A$-call in the terminal block, counted from the left, the
 sequence has the form
\[
A^rB^sA^{j-1}(AB)A^{u-1-j},
\qquad
j=u-1,\ldots,1.
\]
The displayed $AB$ pair is reached after the failure sequence 
$A^rB^sA^{j-1}$, and hence at posterior $\beta_{r+s+j}$. Because this
sequence is unchanged from the original plan, \eqref{eq:D-boundary} gives $D(\lambda_{r+s+j})\leq0.$
Lemma~\ref{lem:bellman-interchange} therefore implies that replacing this
$AB$ pair by $BA$ does not decrease the expected value. Since the
current plan is already optimal, each such interchange must preserve the
optimal value. Sequentially applying these interchanges for
$j=u-1,\ldots,1$ transforms
\[
A^rB^sA^{u-1}B
\quad\text{into}\quad
A^rB^{s+1}A^{u-1}.
\]
 If the resulting terminal $A$-block is nonempty,
the same argument can be applied again using the posterior beliefs along
the new optimal plan. Repeating this procedure until the terminal
$A$-block is empty yields the optimal sequence $A^rB^{s+u}.$
Setting $n_A=r$ and $n_B=s+u$, and noting that $0 \le n_A+n_B =r+s+u\le k $, proves  \eqref{eq:escalator-form-finite proof}.

\vspace{1mm}
\noindent
\underline{Step 2: extension to the infinite-horizon problem.} 
Fix a reachable belief $\beta$. From Step 1,   
for each problem with $n\geq1$ available calls, we have   
an optimal plan in the following form: 
\[
\pi_n=(A^{r_n}B^{s_n};\mathrm{stop}),
\qquad 0\le r_n+s_n\leq n.
\] 
We next consider two cases: $\beta_o=0$ corresponding to Type~1 window, and $\beta_o>0$, corresponding to Type~2 window.

\noindent\underline{Case 1: $\beta_o =0$.}
For finite $r,s\geq0$ and $\theta\in\{e,h\}$, the expected net reward  is 
\begin{align}
V(A^rB^s;\mathrm{stop}\mid\theta)
&= G[1-(1-q_{A,\theta})^r (1-q_{B,\theta})^s] - c_A\sum_{i=0}^{r-1}(1-q_{A,\theta})^i -c_B(1-q_{A,\theta})^r \sum_{i=0}^{s-1}(1-q_{B,\theta})^i   \nonumber\\
&= 
(G-\gamma_{A,\theta})
\bigl[1-(1-q_{A,\theta})^r\bigr]
 +
(1-q_{A,\theta})^r
(G-\gamma_{B,\theta})
\bigl[1-(1-q_{B,\theta})^s\bigr], \label{eq-ArBs}
\end{align} 
where the second identity follows by applying  the definition of $\gamma_{a,\theta}$ and some  simple algebra. We next construct a fixed admissible Escalator plan
$\pi^*$ and prove that
$V_\beta(\pi^*)=J_\infty(\beta)$.

Suppose first that $\{r_n\}$ is unbounded.
Choose an increasing sequence of indices $\{n_j\}$
such that $r_{n_j}\to\infty$, and define
$\pi^*:=A^\infty$.
Then $V(\pi^*\mid\theta) =G-c_A\sum_{i=0}^{\infty}(1-q_{A,\theta})^i=G-\gamma_{A,\theta}.$ By applying \eqref{eq-ArBs} to $\pi_{n_j}= A^{r_{n_j}}B^{s_{n_j}}$, we  further have 
\begin{align*}
 \left|
V(\pi_{n_j}\mid\theta)-V(\pi^*\mid\theta)
\right|
 & =\left|(1-q_{A,\theta})^{r_{n_j}} \big(\gamma_{A,\theta}-\gamma_{B,\theta} - (G-\gamma_{B,\theta})(1-q_{B,\theta})^{s_{n_j}}\big) \right| \nonumber \\
 & \leq
(1-q_{A,\theta})^{r_{n_j}}
\left(
|\gamma_{A,\theta}-\gamma_{B,\theta}|
+
|G-\gamma_{B,\theta}|
\right)
\longrightarrow0, \quad j \to \infty. 
\end{align*} 
Now suppose that $\{r_n\}$ is bounded.
Since $r_n$ is a nonnegative integer, there exists a  finite
integer $r$ which occurs infinitely often.
Choose an increasing sequence of indices $\{n_j\}$
such that $r_{n_j}=r$ for every $j$. If $\{s_{n_j}\}$ is unbounded, pass to a further
subsequence, retaining the notation $\{n_j\}$,
such that $s_{n_j}\to\infty$.
Define the fixed plan $\pi^*:=A^rB^\infty$.
Then  
\begin{align}
    V(\pi^*\mid\theta)
 & =-c_A\sum_{i=0}^{r-1} (1-q_{A,\theta})^i +G-c_B(1-q_{A,\theta})^{r} \sum_{i=0}^{\infty}(1-q_{B,\theta})^i 
\nonumber \\
& = G-\gamma_{A,\theta}
-
(1-q_{A,\theta})^r
(\gamma_{B,\theta}-\gamma_{A,\theta}). \label{eq-ArBinfty}
\end{align}
By applying \eqref{eq-ArBs} to $\pi_{n_j}$, we further have 
\begin{align*}
 \left|
V(\pi_{n_j}\mid\theta)-V(\pi^*\mid\theta)
\right|
  =
(1-q_{A,\theta})^r
|G-\gamma_{B,\theta}|
(1-q_{B,\theta})^{s_{n_j}}
\longrightarrow0, \quad j \to \infty. 
\end{align*} 
If instead $\{s_{n_j}\}$ is bounded, some finite
integer $s$ occurs infinitely often.
Passing to that subsequence and retaining the
notation $\{n_j\}$, define
$\pi^*:=(A^rB^s;\mathrm{stop})$.
Then $\pi_{n_j}=\pi^*$ for every $j$, so  $V(\pi_{n_j}\mid\theta)=V(\pi^*\mid\theta)$.

To conclude, in either case, we have constructed a fixed
admissible Escalator plan $\pi^*$ for the infinite-horizon problem and a subsequence
with $n_j\to\infty$ such that $\lim_{j \to \infty} V(\pi_{n_j}\mid\theta)
= 
V(\pi^*\mid\theta)$ for  $\theta\in\{e,h\}$.  
Since $\beta_o=0$, it follows that
\begin{align*}
V_\beta(\pi^*)
 =
\sum_{\theta\in\{e,h\}}
\beta_\theta V(\pi^*\mid\theta)
 =
\lim_{j\to\infty}
\sum_{\theta\in\{e,h\}}
\beta_\theta V(\pi_{n_j}\mid\theta)
 =
\lim_{j\to\infty}J_{n_j}(\beta)
=
J_\infty(\beta),
\end{align*}
where the third identity follows from the optimality of 
$\pi_{n_j}$, and the last identity follows from Lemma~\ref{lemma-sec 4.2-Bellman infinite horizon}.  
Therefore, $\pi^*$ attains the infinite-horizon
optimal value and has the claimed Escalator structure.

\noindent\underline{Case 2: $\beta_o > 0$.}  
Since every call fails in state $o$, all $r_n+s_n$
prescribed calls are made, and hence $V(\pi_n\mid o)=-(c_A r_n+c_B s_n).$
Using the statewise reward expression  in the first identity of \eqref{eq-ArBs} established in Case~1
for $\theta\in\{e,h\}$, we obtain 
\begin{align*}
    V_\beta(\pi_n)
&=\sum_{\theta\in\{e,h\}}\beta_\theta V(\pi_n\mid\theta)
  +\beta_oV(\pi_n\mid o)\\
&=G\sum_{\theta\in\{e,h\}}\beta_\theta
  \bigl[1-(1-q_{A,\theta})^{r_n}
          (1-q_{B,\theta})^{s_n}\bigr] -c_A\sum_{\theta\in\{e,h\}}\beta_\theta
       \sum_{i=0}^{r_n-1}(1-q_{A,\theta})^i 
 \nonumber\\
 & \quad -c_B\sum_{\theta\in\{e,h\}}\beta_\theta
       (1-q_{A,\theta})^{r_n}
       \sum_{i=0}^{s_n-1}(1-q_{B,\theta})^i 
  -\beta_o(c_A r_n+c_B s_n)\\
&\le
  G\sum_{\theta\in\{e,h\}}\beta_\theta
  \bigl[1-(1-q_{A,\theta})^{r_n}
          (1-q_{B,\theta})^{s_n}\bigr]
  -\beta_o(c_A r_n+c_B s_n)\\
&\le
  G\sum_{\theta\in\{e,h\}}\beta_\theta
  -\beta_o(c_A r_n+c_B s_n) \le 
  G(1-\beta_o)-\beta_o c_A (r_n+  s_n),  
\end{align*}
where  the last inequality uses $\beta_e+\beta_h=1-\beta_o$ and $c_B\ge c_A$.  Since $0\le J_n(\beta)=V_\beta(\pi_n)$ from the optimality of $\pi_n$, we then have  
 \begin{align}
     r_n+s_n
\le \frac{G(1-\beta_o)}{\beta_o c_A},
\qquad n\ge1. \label{ineq-proof-bounded total calls-Type 2} 
 \end{align}
Thus, both $\{r_n\}$ and $\{s_n\}$ are bounded  over  $n$.
The same bounded-integer subsequence argument as in Case~1
yields finite nonnegative integers $r,s$ and a subsequence
with $n_j\to\infty$ such that
 \begin{align}
     \pi_{n_j}=\pi^*:=(A^rB^s;\mathrm{stop})
\qquad\text{for every }j. \label{eq-proof-thm1-infinite-Type 2}
 \end{align}
Therefore, we have the following equation: 
\begin{align*}
    V_\beta(\pi^*) =\sum_{\theta \in \{e,h,o\} }\beta_{\theta} V(\pi^*\mid \theta) =\lim_{j \to \infty} \sum_{\theta \in \{e,h,o\} }\beta_{\theta} V(\pi_{n_j}\mid \theta) 
=
\lim_{j\to\infty}J_{n_j}(\beta)
=
J_\infty(\beta),
\end{align*}
where the second identity follows from \eqref{eq-proof-thm1-infinite-Type 2}, the third identity follows from the optimality of $\pi_{n_j}$ and the definition of $J_{n_j}(\beta)$, and the last identity follows from Lemma~\ref{lemma-sec 4.2-Bellman infinite horizon}. 
Hence, the constructed Escalator plan $\pi^* =(A^rB^s;\mathrm{stop})$ attains the infinite-horizon
optimal value.
This completes the proof of~\eqref{eq:escalator-form} when $\beta_o>0$. 
\hfill \qed

\subsection{Proof of Parts (i) and (ii) in Theorem \ref{thm:escalator}}\label{appendix-corollary-nostop}
\underline{Proof of part (i).}
Since $\mu_o=0$, Bayes' rule implies $\beta_o=0$
at every reachable belief $\beta$.
Repeatedly calling $B$ until success has the following expected
net reward:  
\begin{align}
    V_\beta(B^\infty) = G -\sum_{\theta\in\{e,h\}}\beta_\theta \cdot c_B \sum_{i=0}^{\infty} (1-q_{B,\theta})^i 
=
G-\sum_{\theta\in\{e,h\}}\beta_\theta\gamma_{B,\theta}
\geq G-\gamma_{B,h}
\geq0, \label{non negative reward-repeated call B}
\end{align}
where the two inequalities follow from
Assumption~\ref{ass:two-model-structure}(i) and (iv), respectively.
Since immediate stopping  yields zero, stopping before success
is never strictly optimal.

By the structural result already established in Appendix~\ref{appendix-thm-escalator},
an optimal plan is either a finite Escalator plan $(A^rB^s;\text{stop})$ for finite $r$ and $s$,
$A^rB^\infty$, or $A^\infty$.
If the optimal plan is finite, replace its terminal
stop by repeated calls to $B$ until success.
Conditional on reaching that stop, the additional
expected net reward is nonnegative by  
inequality~\eqref{non negative reward-repeated call B} applied at the corresponding posterior.
The extended plan therefore remains optimal and
has the form $A^rB^\infty$.
Thus, an optimal plan can be chosen with $n_A=\infty$ or $n_B=\infty$, with $n_A=\infty$ interpreted
as $A^\infty$.

Finally, suppose that $\gamma_{B,h}<\gamma_{A,h}$. Note that by the definition of the prior $\mu$, $\mu_h>0$.
Bayes' rule then implies $\beta_h>0$ at every
reachable belief.
For each finite non-negative integer $r$, by applying \eqref{eq-ArBinfty} and \eqref{non negative reward-repeated call B} modified to $A$, we have 
\begin{align*}
 V_\beta(A^rB^\infty)-V_\beta(A^\infty) & =
\sum_{\theta\in\{e,h\}}
\beta_\theta(1-q_{A,\theta})^r
(\gamma_{A,\theta}-\gamma_{B,\theta})
\\
& =
(1-q_{A,h})^r
\left[
\beta_h(\gamma_{A,h}-\gamma_{B,h})
+
\beta_e
\left(\frac{1-q_{A,e}}{1-q_{A,h}}\right)^r
(\gamma_{A,e}-\gamma_{B,e})
\right].
\end{align*}
By Assumption~\ref{ass:two-model-structure}(i),
$0<(1-q_{A,e})/(1-q_{A,h})<1$.
The expression in brackets therefore converges to
$\beta_h(\gamma_{A,h}-\gamma_{B,h})>0$ as $r \to \infty$.
Since $(1-q_{A,h})^r>0$ for every finite $r$,
the whole  difference is strictly positive
for all sufficiently large finite $r$. 
Since an optimal non-stopping Escalator plan exists by the preceding
argument and $A^\infty$ is not optimal, such a plan has the form
$A^rB^\infty$ for finite $r$. Thus, $n_A$ can be chosen to be finite.

\vspace{1mm}
\noindent\underline{Proof of part (ii).}  
Let $\pi$ be any optimal deterministic plan, with $N$ calls along
its all-failure path. Since $\mu_o>0$, an infinite plan has value
$-\infty$ and cannot be optimal. Thus, $N$ is finite. Success is
impossible in state $o$, where all scheduled calling costs are
incurred. Since immediate stopping yields zero,
\[
0\le V_\mu(\pi)
\le G(1-\mu_o)-\mu_o\sum_{i=1}^N c_{a_i}
\le G(1-\mu_o)-\mu_o c_A N.
\]
Hence,
$N\le\lfloor G(1-\mu_o)/(\mu_o c_A)\rfloor$,
as claimed. 
 \hfill \qed  

\subsection{Proof of Corollary \ref{corollary-positive calls}}\label{appendix-corollary-escalator} 

We first show that an optimal Escalator plan can be
chosen to have $n_A\ge 1$.
For any nonempty $B$-only plan, including $B^\infty$,
the expected  reward from certified improvement is at most $G(1-\mu_o)$,
and the first call incurs cost $c_B$.
Its expected net reward is therefore at most
$G(1-\mu_o)-c_B$.
By assumption,
\[
V_\mu(A;\mathrm{stop})
=
Gp_A(\mu)-c_A
\ge
\max\{0,\;G(1-\mu_o)-c_B\}.
\]
Thus, one call to $A$ followed by stopping weakly
dominates both immediate stopping and every $B$-only plan.
By Theorem~\ref{thm:escalator}, an optimal Escalator plan exists.
If its $A$-block is empty, replacing it by
$(A;\mathrm{stop})$ preserves optimality.
Hence, an optimal Escalator plan can be chosen
to begin with $A$.

\vspace{1mm}
\noindent
\underline{Proof of part (i).}
Consider a Type~1 window and an optimal plan
beginning with $A$.
Its continuation after the first call fails must be
optimal at $T_A(\mu)$. 
Under Assumption~\ref{ass:two-model-structure}(iv) and
$\gamma_{B,h}<\gamma_{A,h}$, Theorem~\ref{thm:escalator}(i), applied at
$T_A(\mu)$, implies that this continuation can be
chosen as $A^rB^\infty$ for some finite $r\ge0$.
Combining it with the initial call to  $A$ gives the optimal
plan $A^{r+1}B^\infty$, proving part (i).

\vspace{1mm}
 \noindent 
\underline{Proof of part (ii).}
Consider a Type~2 window.
By Theorem~\ref{thm:escalator}(ii), the optimal plan  
is finite.
If it contains a call to $B$, the result directly holds. 
We next consider an optimal plan with only calls to $A$.

We first show that $(A^{n_0};\mathrm{stop})$ is  optimal. By the definition of $p_A(\beta)$ in equation~\eqref{eq:posterior-success}    and the Bayesian update rule in  equation~\eqref{eq:failure-update}, 
\begin{align*}
    p_A(T_A(\beta))
 =\sum_{\theta\in\Theta}
(T_A(\beta))_\theta q_{A,\theta} =
\frac{
\sum_{\theta\in\Theta}
\beta_\theta(1-q_{A,\theta})q_{A,\theta}
}{
1-p_A(\beta)
}  =
\frac{
p_A(\beta)
-\sum_{\theta\in\Theta}\beta_\theta q_{A,\theta}^2
}{
1-p_A(\beta)
}.
\end{align*}
Consequently, the difference in the posterior success probabilities has the following expression: 
\[
\begin{aligned}
p_A(\beta)-p_A(T_A(\beta))
& =
p_A(\beta)
-
\frac{
p_A(\beta)
-\sum_{\theta\in\Theta}\beta_\theta q_{A,\theta}^2
}{
1-p_A(\beta)
} \nonumber \\
& =
\frac{
\sum_{\theta\in\Theta}\beta_\theta q_{A,\theta}^2
-(p_A(\beta))^2
}{
1-p_A(\beta)
} =
\frac{
\sum_{\theta\in\Theta}
\beta_\theta(q_{A,\theta}-p_A(\beta))^2
}{
1-p_A(\beta)
}
\ge0,
\end{aligned}
\]
where the last equality uses
$\sum_{\theta\in\Theta}\beta_\theta=1$ and
$\sum_{\theta\in\Theta}\beta_\theta q_{A,\theta}=p_A(\beta)$,
and the inequality follows from $1-p_A(\beta)>0$.
Thus, $p_A(T_A(\beta))\le p_A(\beta)$. By repeatedly applying this argument, we conclude that  $p_A(T_A^n(\mu))$ is non-increasing in $n$.   
Note that 
 \begin{align*}
     p_A(T_A^n(\mu))
=\frac{
\sum_{\theta\in\{e,h\}}
\mu_\theta q_{A,\theta}(1-q_{A,\theta})^n
}{
\sum_{\theta\in\Theta}
\mu_\theta(1-q_{A,\theta})^n
} = 
\frac{
\sum_{\theta\in\{e,h\}}
\mu_\theta q_{A,\theta}(1-q_{A,\theta})^n
}{
\mu_o+
\sum_{\theta\in\{e,h\}}
\mu_\theta(1-q_{A,\theta})^n
}
\longrightarrow0.
 \end{align*}
Since $\mu_o > 0$ for Type~2 window, $n_0$ defined in part (ii) is finite. 
Then for every integer $n\ge0$, appending one $A$-call
to $(A^n;\mathrm{stop})$ changes its expected net
reward by
\[
\begin{aligned}
 V_\mu(A^{n+1};\mathrm{stop})
 -V_\mu(A^n;\mathrm{stop}) 
 =
\left[
\sum_{\theta\in\Theta}
\mu_\theta(1-q_{A,\theta})^n
\right]
\left[
Gp_A(T_A^n(\mu))-c_A
\right].
\end{aligned}
\] 
By the definition of $n_0$ and the monotonicity
established above, the second bracket in the RHS of the above equation  is non-negative
for $0\le n<n_0$ and non-positive for $n\ge n_0$.
Consequently, the sequence
$V_\mu(A^n;\mathrm{stop})$ is non-decreasing up to
$n=n_0$ and non-increasing thereafter.
Thus, $(A^{n_0};\mathrm{stop})$ is optimal among
all finite $A$-only plans.
Since an $A$-only plan is globally optimal in the
case under consideration, we have $V_\mu(A^{n_0};\mathrm{stop})=J_\infty(\mu).$

Finally, consider adding one call to $B$ if   $n_0$ calls to $A$ all fail.
The resulting change in expected net reward is
\begin{align*} V_\mu(A^{n_0}B;\mathrm{stop})
 -V_\mu(A^{n_0};\mathrm{stop}) =
\left[
\sum_{\theta\in\Theta}
\mu_\theta(1-q_{A,\theta})^{n_0}
\right]
\left[
Gp_B(T_A^{n_0}(\mu))-c_B
\right]
\ge0, 
\end{align*}
where the inequality follows from the condition
in part (ii).
Therefore, $(A^{n_0}B;\mathrm{stop})$ is also
optimal. Since $n_0\ge1$, both blocks are nonempty,
which proves part (ii). \hfill \qed

\section{Proofs of the Competitive Guarantee (Section \ref{sec:competitive})}\label{ec:proofs-competitive}

\subsection{Proof of Proposition~\ref{prop:prior-free-escalator}}\label{appendix-prop-structure}
The claim is immediate when $J\le1$. Otherwise, fix $1 \le   j<J$.
By the maximizing property of the selected configurations,
\begin{align*}
R_{\alpha_j}(B_j)
 \ge R_{\alpha_{j+1}}(B_j), \qquad 
R_{\alpha_{j+1}}(B_{j+1})
 \ge R_{\alpha_j}(B_{j+1}).
\end{align*}
Applying the definition of $R_a(B)$ and dividing the above inequalities by $B_j$ and $B_{j+1}$,
respectively, we obtain 
\begin{align} s_{\alpha_j}^*\left(\frac{1}{c_{\alpha_j}}- \frac{1}{B_j}\right) \ge s_{\alpha_{j+1}}^*\left(\frac{1}{c_{\alpha_{j+1}}}- \frac{1}{B_j}\right), \quad s_{\alpha_{j+1}}^*\left(\frac{1}{c_{\alpha_{j+1}}}- \frac{1}{B_{j+1}}\right) \ge s_{\alpha_{j}}^*\left(\frac{1}{c_{\alpha_{j}}}- \frac{1}{B_{j+1}}\right). \label{ineq-compare-escalator structure} 
\end{align}
Combining the above two inequalities, we have
\[
\left(\frac{1}{B_j}-\frac{1}{B_{j+1}}\right)
\left(s_{\alpha_{j+1}}^*-s_{\alpha_j}^*\right)\ge0.
\]
Since $B_{j+1}>B_j$, it follows that
$s_{\alpha_{j+1}}^*\ge s_{\alpha_j}^*$.
If $s_{\alpha_{j+1}}^*= s_{\alpha_j}^*$, since $\mathbf s^* >0$ by its definition,  the two inequalities in \eqref{ineq-compare-escalator structure} imply $c_{\alpha_j} \le c_{\alpha_{j+1}}$ and $c_{\alpha_{j+1}} \le c_{\alpha_{j}}$, respectively. Therefore, $c_{\alpha_j} = c_{\alpha_{j+1}}$ and  $\alpha_j=\alpha_{j+1}$, which does not affect the escalator structure. 

Now consider $s_{\alpha_{j+1}}^* > s_{\alpha_j}^*$.  Since $R_{\alpha_j}(B_j) \ge R_1(B_j) = 2^j -1 >0$, from the definition of $R_{\alpha_j}(B_j)$, we see that   $B_j>c_{\alpha_j}$.
Suppose $c_{\alpha_{j+1}}\le c_{\alpha_j}$. Then $B_j>c_{\alpha_{j+1}}$ and 
\begin{align*}
R_{\alpha_{j+1}}(B_j)
 =
s_{\alpha_{j+1}}^*
\left(\frac{B_j}{c_{\alpha_{j+1}}}-1\right)   > 
s_{\alpha_j}^*
\left(\frac{B_j}{c_{\alpha_j}}-1\right)
=
R_{\alpha_j}(B_j),
\end{align*}
leading to contradiction with the optimality of $\alpha_j$ in maximizing $R_a(B_j)$ over $a \in \mathcal{A}$.
Therefore, $c_{\alpha_{j+1}}>c_{\alpha_j}$, and the cost
ranking implies $\alpha_{j+1}>\alpha_j$. Since the analysis applies to any $j=1,2,\ldots, J-1$,  the returned
routing plan is an Escalator policy. \hfill \qed

\subsection{Proof of Theorem~\ref{thm:escalator-guarantee}}\label{appendix-thm-guarantee}
Fix $\theta\in\mathcal D$ and define
\[
\tau_\theta
:=
\min_{a\in\mathcal A}
\frac{-\log(1-q_{a,\theta})}{s_a^*}>0.
\]
Recall that $-\log(1-q_{a,\theta})$ is the additive
contribution of one $a$-call to the negative logarithm
of the all-failure probability.
Thus, $\tau_\theta$ is the smallest such contribution
per unit weight across configurations in state $\theta$.
It provides a common conservative rate for the auxiliary
model constructed below.   
By the definition of $\kappa_q(\mathbf s^*)$, we have
\begin{equation}
\tau_\theta s_a^*
\le
-\log(1-q_{a,\theta})
\le
\kappa_q(\mathbf s^*)\tau_\theta s_a^*,
\qquad a\in\mathcal A.
\label{eq:ep-log-comparison}
\end{equation}

We construct an auxiliary problem with success probabilities
$\widetilde q_{a,\theta}:=1-e^{-\tau_\theta s_a^*}$
for $a\in\mathcal A$ and $\theta \in  {\mathcal D}$, while still assuming conditional independence
across calls.
This replaces each configuration's logarithmic contribution
by $\tau_\theta$ times its weight, so the contribution
per unit weight is identical across configurations.
In particular, for any prescribed sequence
$a_1,\ldots,a_n$ with total weight
$u=\sum_{i=1}^n s_{a_i}^*$,
\[
\prod_{i=1}^n(1-\widetilde q_{a_i,\theta})
=
\exp\left(-\tau_\theta\sum_{i=1}^n s_{a_i}^*\right)
=
e^{-\tau_\theta u}.
\]
Thus, the auxiliary all-failure probability depends only
on the cumulative weight, rather than on the particular
configurations used. This exponential form allows us
to represent success through a single random weight
threshold later in the proof.  
Let $\widetilde V(\pi^{\rm EP}\mid\theta)$ denote the expected
net reward of the same Escalator policy in Algorithm~\ref{alg:det-escalator} when each success probability
$q_{a,\theta}$ is replaced by $\widetilde q_{a,\theta}$.
The reward $G$, the call costs, and the rule of stopping at the
first success remain the same.
We then make the following observation: 
\begin{equation}
G-V(\pi^{\rm EP}\mid\theta)
\le
G-\widetilde V(\pi^{\rm EP}\mid\theta), \qquad \forall \theta \in \mathcal{D}. 
\label{eq:ep-value-comparison}
\end{equation}
To see \eqref{eq:ep-value-comparison}, note that  the first inequality in \eqref{eq:ep-log-comparison} gives
 \begin{align}
     1-q_{a,\theta}
\le
e^{-\tau_\theta s_a^*}
=
1-\widetilde q_{a,\theta}.\label{ineq-comparison-fail probability}
 \end{align}
By applying the expression of $V(\cdot\mid\theta)$ in equation~\eqref{eq:statewise-value}, we obtain 
\begin{align*}
    G-V(\pi^{\rm EP}\mid\theta)
 =
G\prod_{j=1}^J
(1-q_{\alpha_j,\theta})^{n_j}
 +
\sum_{j=1}^J c_{\alpha_j}
\sum_{\ell=1}^{n_j}
\left[
\prod_{k=1}^{j-1}
(1-q_{\alpha_k,\theta})^{n_k}
\right]
(1-q_{\alpha_j,\theta})^{\ell-1},
\end{align*}
which increases in $1-q_{a,\theta}$. This, combined with~\eqref{ineq-comparison-fail probability}, implies~\eqref{eq:ep-value-comparison}. 

\vspace{1mm}
Before proceeding to bound $G-\widetilde V(\pi^{\text{EP}}|\theta)$, we introduce $\Phi(u)$ and $\Psi(u)$.
For every $u\ge0$, let
\[
\Phi(u)
:=
\min_{a\in\mathcal A}
c_a\left(1+\frac{u}{s_a^*}\right).
\]
For a fixed configuration $a$, the expression
$c_a(1+u/s_a^*)$ is an upper bound on the total cost of the
$\lceil u/s_a^*\rceil$ calls needed to attain a total weight
of at least $u$. 
When $J=0$, define $\Psi(u):=G$ for all $u\ge0$, and when $J\ge1$, define
\[
\begin{aligned}
\Psi(u)
:={}&
\sum_{j=1}^J c_{\alpha_j}
\sum_{\ell=1}^{n_j}
\mathbf 1\left\{
\sum_{k=1}^{j-1}n_k s_{\alpha_k}^*
+
(\ell-1)s_{\alpha_j}^*
\le u
\right\} +
G\mathbf 1\{u\ge R(B_J)\},
\qquad u\ge0.
\end{aligned}
\]
The double sum includes the cost of each scheduled call whose
starting cumulative weight is at most $u$.
The last term assigns a loss of $G$ once $u$ reaches the
final target $R(B_J)$; it is used to bound the reward lost
when the plan fails.

\vspace{1mm}
We next relate $\Psi$ to the auxiliary expected net reward. Let  $U$ be an exponential random variable with mean $1/\tau_{\theta}$ and we  establish the following inequality:  
\begin{align}
G-\widetilde V(\pi^{\rm EP}\mid\theta) \le  \mathbb{E}[\Psi(U)]. \label{eq:ep-aux-shortfall}
\end{align} 
To understand the above inequality intuitively, since $\Pr(U>u)=e^{-\tau_\theta u}$ equals the auxiliary probability
that calls with total weight $u$ all fail, $U$ can be interpreted as
a random weight threshold for success. A scheduled call is reached
when its starting cumulative weight is below $U$, so the double sum
in $\Psi(U)$ records the corresponding calling cost. The term
$G\mathbf 1\{U\ge R(B_J)\}$ upper-bounds the reward lost when the
plan terminates without success. Hence, $\Psi(U)$ upper-bounds the
realized loss, consisting of calling cost plus forgone reward.
Taking expectations gives~\eqref{eq:ep-aux-shortfall}. 
To formally prove  \eqref{eq:ep-aux-shortfall},  when $J=0$, both sides equal $G$, so
\eqref{eq:ep-aux-shortfall} \textcolor{black}{naturally} holds. 
When $J\ge1$, under the auxiliary success probabilities, the probability
of reaching the $\ell$th call in stage $j$ is $\exp(
-\tau_\theta
[
\sum_{k=1}^{j-1}n_k s_{\alpha_k}^*
+
(\ell-1)s_{\alpha_j}^*
]
)$. Therefore, we have
\begin{align}
G-\widetilde V(\pi^{\rm EP}\mid\theta)
&= 
\sum_{j=1}^J c_{\alpha_j}
\sum_{\ell=1}^{n_j}
\exp\left(
-\tau_\theta
\left[
\sum_{k=1}^{j-1}n_k s_{\alpha_k}^*
+
(\ell-1)s_{\alpha_j}^*
\right]
\right) +
G\exp\left(
-\tau_\theta\sum_{j=1}^J n_j s_{\alpha_j}^*
\right) \nonumber \\
&\le \sum_{j=1}^J c_{\alpha_j}
\sum_{\ell=1}^{n_j} \mathbb{P}\left( 
\sum_{k=1}^{j-1}n_k s_{\alpha_k}^*
+
(\ell-1)s_{\alpha_j}^*\le U\right) + G\exp\left(
-\tau_\theta\sum_{j=1}^J n_j s_{\alpha_j}^*
\right) \nonumber \\
&\le \sum_{j=1}^J c_{\alpha_j}
\sum_{\ell=1}^{n_j} \mathbb{P}\left( 
\sum_{k=1}^{j-1}n_k s_{\alpha_k}^*
+
(\ell-1)s_{\alpha_j}^*\le U\right) +\mathbb E[G\mathbf 1\{U\ge R(B_J)\}]\nonumber\\
& = \mathbb{E}[\Psi(U)], \nonumber  
\end{align} 
where   the first inequality follows by applying expression of  the cumulative distribution function of $U$, and 
the second inequality follows from $\sum_{j=1}^J n_js_{\alpha_j}^*\ge R(B_J)$, as guaranteed by the updating rule. 
 This completes the proof of inequality~\eqref{eq:ep-aux-shortfall}.

 \vspace{1mm}
We now prove that
\begin{align}
\Psi(u)\le 4\min\{\Phi(u),G\},
\qquad u\ge0.
\label{eq:ep-pointwise-coverage}
\end{align}
Since $\Phi(u)\ge c_1$, when   $J=0$, we have  
$G<2c_1$ and $\Psi(u)=G\le 2\min\{\Phi(u),G\}$.  
Now suppose $J\ge1$, so that $B_J\le G<2B_J$  from the definition of $J$ and $B_J$.  
Note that when $n_j>0$,
the ceiling rule and the nonnegativity of the weight accumulated
before stage $j$ give
\[
n_jc_{\alpha_j}
\le c_{\alpha_j}\left(1+\frac{R(B_j)}{s_{\alpha_j}^*}\right) = c_{\alpha_j}\left(1+\frac{R_{\alpha_j}(B_j)}{s_{\alpha_j}^*}\right) = c_{\alpha_j}\left(1+\frac{s_{\alpha_j}^*(B_j/c_{\alpha_j}-1)}{s_{\alpha_j}^*}\right)
=B_j. 
\]
The same bound is immediate when $n_j=0$. Consequently,
\begin{align}
\sum_{k=1}^{j}n_kc_{\alpha_k}
\le \sum_{k=1}^{j}B_k
=2B_j-2c_1<2B_j,
\qquad j=1,\ldots,J.
\label{eq:ec14}
\end{align}
Also, the definitions of $R(\cdot)$ and $\Phi(\cdot)$ imply the following  relationship: 
\begin{align}
u\ge R(B) \iff
u\ge \max_{a\in \mathcal{A}}s_a^*\left(\frac{B}{c_a}-1\right) 
 \iff
c_a\left(1+\frac{u}{s_a^*}\right)\ge B, \, \forall a\in\mathcal A \iff
\Phi(u)\ge B.\label{eq:ep-reach-inverse}
\end{align} 
We distinguish two cases according to the value of $\Phi(u)$.
\begin{itemize}
    \item If $\Phi(u)<B_J$, then $\Phi(u)\ge c_1$ and the doubling
of the budgets ensure that there exists $j\in\{1,\ldots,J\}$
such that $B_j/2=B_{j-1}\le\Phi(u)<B_j.$
By \eqref{eq:ep-reach-inverse}, $u<R(B_j)\le R(B_J)$, so the
$G$ term in $\Psi(u)$ vanishes. The cumulative
weight at the end of stage $j$ is at least $R(B_j)>u$,
so no call in a later stage is included in $\Psi(u)$.
Using \eqref{eq:ec14} and $\Phi(u)<B_J\le G$, we obtain
\begin{align*}
\Psi(u) \le \sum_{k=1}^{j}n_kc_{\alpha_k}<2B_j \le 4\Phi(u)=4\min\{\Phi(u),G\}.
\end{align*}
\item If $\Phi(u)\ge B_J$, then both $\Phi(u)$ and $G$ are
at least $B_J$. Bounding $\Psi(u)$ by the total scheduled
cost plus $G$ gives
\begin{align*}
\Psi(u)
 \le \sum_{j=1}^{J}n_jc_{\alpha_j}+G
<2B_J+G <4B_J\le4\min\{\Phi(u),G\}.
\end{align*}
This completes the proof of \eqref{eq:ep-pointwise-coverage}. 
\end{itemize}

\vspace{1mm}
We now establish the following upper bound on $\mathbb E[\Phi(U)]$:
\begin{equation}
\mathbb E[\Phi(U)]
\le
h\bigl(\kappa_q(\mathbf s^*)\bigr)
\min_{a\in\mathcal A}\gamma_{a,\theta}.
\label{eq:ep-comparison-expectation}
\end{equation}
For every $a\in\mathcal A$, the definition of $\Phi$ gives $\mathbb E[\Phi(U)]
\le
c_a (
1+{\mathbb E[U]}/{s_a^*}
 )
=
c_a (
1+ {1}/{(\tau_\theta s_a^*)}
).$
The second inequality in \eqref{eq:ep-log-comparison}
implies $q_{a,\theta}
\le
1-e^{-\kappa_q(\mathbf s^*)\tau_\theta s_a^*}.$
Recalling $\gamma_{a,\theta}=c_a/q_{a,\theta}$, we obtain
\[
\begin{aligned}
\mathbb E[\Phi(U)]
&\le
\gamma_{a,\theta}q_{a,\theta}
\left(
1+\frac{1}{\tau_\theta s_a^*}
\right)
 \le
\gamma_{a,\theta}
\left(
1-e^{-\kappa_q(\mathbf s^*)\tau_\theta s_a^*}
\right)
\left(
1+\frac{1}{\tau_\theta s_a^*}
\right) 
 \le
\gamma_{a,\theta}
h\bigl(\kappa_q(\mathbf s^*)\bigr),
\end{aligned}
\]
where the last inequality follows 
by applying the definition of $h(\cdot)$ and using the fact that $\kappa_q(\mathbf s^*) \ge 1$. 
Since this bound holds for every configuration, inequality~\eqref{eq:ep-comparison-expectation} holds.

\vspace{1mm}
We can now combine these inequalities to bound $G-V(\pi^{\text{EP}}|\theta)$ for any $\theta\in \mathcal{D}$.
Using \eqref{eq:ep-value-comparison},
\eqref{eq:ep-aux-shortfall},
\eqref{eq:ep-pointwise-coverage}, and
\eqref{eq:ep-comparison-expectation}, we obtain
\[
\begin{aligned}
G-V(\pi^{\rm EP}\mid\theta)
&\le
G-\widetilde V(\pi^{\rm EP}\mid\theta)
 \le
\mathbb E[\Psi(U)]
 \le
4\mathbb E[\min\{\Phi(U),G\}]
 \le
4\min\{\mathbb E[\Phi(U)],G\}
\\
&\le
4\min\left\{
h\bigl(\kappa_q(\mathbf s^*)\bigr)
\min_{a\in\mathcal A}\gamma_{a,\theta},
G
\right\}
 \le
4h\bigl(\kappa_q(\mathbf s^*)\bigr)
\min\left\{
\min_{a\in\mathcal A}\gamma_{a,\theta},
G
\right\}
 =
4h\bigl(\kappa_q(\mathbf s^*)\bigr)
\bigl(G-V^{\rm CL}(\theta)\bigr)\textcolor{black}{,}
\end{aligned}
\]
where the fourth inequality holds because the random minimum is bounded above by both
$\Phi(U)$ and $G$, the last inequality holds because  
$h(\kappa_q(\mathbf s^*))\ge1$, and the final equality
follows from equation~\eqref{eq-CL-reward}.

It remains to consider the non-improvable state $o$.
All scheduled calls fail in this state, so $V(\pi^{\rm EP}\mid o)
=
-\sum_{j=1}^J n_jc_{\alpha_j}.$
When $J\ge1$, \eqref{eq:ec14} and $B_J\le G$  imply  $\sum_{j=1}^J n_jc_{\alpha_j} \le 2B_J<2G.$
The same inequality holds when $J=0$, since the sum is zero.
As $V^{\rm CL}(o)=0$, we conclude that
\[
\begin{aligned}
G-V(\pi^{\rm EP}\mid o)
 =
G+\sum_{j=1}^J n_jc_{\alpha_j}
 <
3G
 \le
4h\bigl(\kappa_q(\mathbf s^*)\bigr)
\bigl(G-V^{\rm CL}(o)\bigr).
\end{aligned}
\]

The desired comparison thus holds for every
$\theta\in\Theta$.
Dividing by the positive quantity
$G-V^{\rm CL}(\theta)$ and taking the supremum over
$\theta$ yields $\operatorname{CR}(\pi^{\rm EP})
\le
4h\bigl(\kappa_q(\mathbf s^*)\bigr),$
which proves the theorem. \hfill \qed

\section{Details of Numerical Study (Section \ref{sec:numerics})}\label{ec:numerics}

\subsection{Models, Menu and Settings}\label{ec:num-models}

\textcolor{black}{We access the Qwen \textcolor{black}{LLMs} described in Section~\ref{sec:numerics} through Alibaba Cloud Model Studio as \texttt{qwen3.5-flash} and \texttt{qwen3.6-flash} \citep{qwen2026qwen35, qwen2026qwen36}. Both run with thinking enabled. Thinking tokens account for about four fifths of output tokens and are billed as output in all reported costs.}

\textcolor{black}{\textcolor{black}{The configuration menu and attempt counts follow Section~\ref{sec:num-main}, with restart interval $r=3$ in Type~1 windows.} Qwen3.5-Flash constructs the parameter schema and handles the first two syntax-repair rounds. Qwen3.6-Flash writes the initial formulation and handles subsequent repair rounds. An attempt is abandoned if the same error recurs three times. COPT runs on one thread with a memory cap per solve.}

\subsection{Baselines}\label{ec:num-baselines}

\textcolor{black}{Each one-shot repetition uses OSCAR's initial Coder prompts, followed by syntax repair, and returns the first runnable program's output. It has no Simulator, Reviewer or second formulation.} \textcolor{black}{Unlike OSCAR, both one-shot models use Qwen3.5-Flash for at most ten repair rounds, so only the initial Coder model varies.} \textcolor{black}{A program that still fails counts as a failed repetition. The structure constructor is called and charged for every repetition. We run 100 repetitions per model and problem.}

\textcolor{black}{For each problem, we run five independent agents with Claude Code using Fable~5.1 and five with Codex using GPT-6 Astra. Each receives only the description and instance and must write and run a COPT program that returns a solution in the required format. The agents cannot access the trusted checker, reference solution, Simulator or feedback. We grade their final solutions with the trusted checker and report the outcomes in Table~\ref{tab:num-main}. As discussed in Section~\ref{sec:num-main}, these runs primarily compare costs. In earlier testing, Fable~5.0 failed on one problem because of a syntax error, which Fable~5.1 did not repeat.}

\subsection{Cost Accounting}\label{ec:num-cost}

\textcolor{black}{We use each model and platform's list rates. At Alibaba Cloud's Beijing endpoint, prices per million input and output tokens are 0.2 and 2 RMB for Qwen3.5-Flash, and 1.2 and 7.2 RMB for Qwen3.6-Flash.} \textcolor{black}{We convert at 6.91 RMB per dollar, the rounded average of daily Federal Reserve H.10 rates from September 12, 2025 through September 11, 2026.} \textcolor{black}{Fable~5.1 and GPT-6 Astra use their providers' dollar rates. Cached input is charged at each provider's cached-input rate.} \textcolor{black}{GPT-6 Astra's rates per million tokens are \$10 for uncached input, \$1 for cached input, \$12.50 for cache writes and \$50 for output. Agent costs use token counts reported in their transcripts.}

\textcolor{black}{Each Simulator is built once per problem, but every run is charged its full build cost to represent a run from scratch. The repeated runs therefore evaluate OSCAR conditional on the constructed Simulator. Table~\ref{tab:num-cost} gives the cost components. The stronger Coder is the largest recurring expense, consistent with Section~\ref{sec:observations}. Although a weaker-model Reviewer call costs more than a Coder call, upgrading the Coder costs more than upgrading the Reviewer, preserving the configuration cost order.}

\begin{table}[htbp]
\centering
\small
\setlength{\tabcolsep}{3pt}
\caption{\textcolor{black}{Mean cost in dollars (model calls) per OSCAR run \textcolor{black}{at $n=1$ unless indicated otherwise}. Initial formulation includes its syntax repair. Loop costs include the Type~1 extension. Totals may differ from component sums because of rounding.}}\label{tab:num-cost}
\begin{tabular}{llcccccc}
\toprule
Component & Model & \textcolor{black}{\shortstack{Routing\\$n=1$}} & \textcolor{black}{\shortstack{Routing\\$n=(1,1,3,3)$}} & Graph & Packing & Scheduling & Quadratic \\
\midrule
Simulator build     & both    & 0.025 (9.0)  & \textcolor{black}{0.025 (9.0)} & 0.007 (3.0)  & \textcolor{black}{0.084} (\textcolor{black}{16.0}) & 0.075 (11.0) & 0.062 (14.0) \\
Structure           & Qwen3.5 & 0.002 (1.0)  & \textcolor{black}{0.002 (1.0)} & 0.001 (1.0)  & 0.002 (1.0)  & 0.002 (1.0)  & 0.001 (1.0) \\
Initial formulation & both    & 0.052 (4.4)  & \textcolor{black}{0.043 (4.0)} & 0.007 (1.9)  & 0.033 (4.2)  & 0.051 (3.8)  & 0.020 (3.6) \\
Reviewer            & Qwen3.5 & 0.010 (3.5)  & \textcolor{black}{0.021 (7.1)} & 0.005 (2.1)  & 0.011 (3.8)  & 0.009 (3.4)  & 0.007 (2.2) \\
Reviewer            & Qwen3.6 & 0.023 (2.7)  & \textcolor{black}{0.046 (5.1)} & 0.008 (2.0)  & 0.048 (6.0)  & 0.025 (3.4)  & 0.014 (2.1) \\
Coder               & Qwen3.5 & 0.007 (3.8)  & \textcolor{black}{0.008 (4.7)} & 0.002 (2.1)  & 0.006 (3.9)  & 0.005 (3.5)  & 0.003 (2.2) \\
Coder               & Qwen3.6 & 0.034 (2.3)  & \textcolor{black}{0.108 (7.4)} & 0.014 (2.0)  & 0.065 (5.8)  & 0.044 (3.2)  & 0.018 (2.0) \\
Syntax repair       & both    & 0.016 (4.1)  & \textcolor{black}{0.026 (6.2)} & 0.003 (1.1)  & 0.018 (5.1)  & 0.028 (6.6)  & 0.009 (2.6) \\
\midrule
Total               &         & 0.168 (30.9) & \textcolor{black}{0.279 (44.4)} & 0.048 (15.3) & \textcolor{black}{0.267} (\textcolor{black}{45.8}) & 0.240 (35.9) & 0.134 (29.7) \\
\bottomrule
\end{tabular}
\end{table}

\end{document}